%% file: paper_arxiv.tex
\documentclass{article}

\usepackage{arxiv}

\usepackage[utf8]{inputenc} 
\usepackage[T1]{fontenc}    
\usepackage{hyperref}       
\usepackage{url}            
\usepackage{booktabs}       
\usepackage{amsfonts}       
\usepackage{nicefrac}       
\usepackage{microtype}      
\usepackage{lipsum}		
\usepackage{graphicx}
\usepackage{natbib}
\usepackage{doi}

\input{preamble/preamble}

\title{Sequential Flow Straightening for Generative Modeling}
\author{
	Jongmin Yoon \\
	Kim Jaechul Graduate School of AI\\
	Daejeon, Korea \\
	\texttt{jm.yoon@kaist.ac.kr}
	\And
	Juho Lee \\
	Kim Jaechul Graduate School of AI\\
	Daejeon, Korea \\
	\texttt{juholee@kaist.ac.kr} \\
}
\renewcommand{\shorttitle}{Sequential Flow Matching}

\begin{document}
\maketitle



\begin{abstract}
\input{main/00_abstract}
\end{abstract}

\input{main/01_introduction}
\input{main/02_backgrounds}
\input{main/03_methods}
\input{main/04_related_works}

\input{main/05_experiments}
\input{main/06_conclusion}

\bibliographystyle{unsrtnat}
\bibliography{references}

\newpage
\appendix
\onecolumn

\input{appendix/additional_experiments}


\end{document}

%% file: preamble/preamble.tex
\usepackage{amsmath, amssymb, amsthm}
\usepackage{mathtools}
\usepackage{bbm}
\usepackage{dsfont}

\usepackage{graphicx}

\usepackage{booktabs, array}

\usepackage[nameinlink,capitalise, noabbrev]{cleveref}
\creflabelformat{equation}{#2\textup{#1}#3}  
\crefname{section}{\S}{\S\S}

\input{preamble/acronyms.tex}
\input{preamble/math.tex}

\usepackage[usenames,dvipsnames]{xcolor}
\newcommand{\red}[1]{\textcolor{BrickRed}{#1}}

\newcommand{\green}[1]{\textcolor{OliveGreen}{#1}}
\newcommand{\blue}[1]{\textcolor{NavyBlue}{#1}}

\usepackage{algorithm}
\usepackage{algorithmic}
\usepackage{eqparbox}
\renewcommand{\algorithmiccomment}[1]{\hfill\eqparbox{COMMENT}{(#1)}}

\usepackage{caption}
\usepackage{subcaption}
\newcommand{\yellow}[1]{\textcolor{Apricot}{#1}}
\usepackage{wrapfig}
\usepackage{listings}

\definecolor{codegreen}{rgb}{0,0.6,0}
\definecolor{codegray}{rgb}{0.5,0.5,0.5}
\definecolor{codepurple}{rgb}{0.58,0,0.82}
\definecolor{backcolour}{rgb}{0.95,0.95,0.92}
\lstdefinestyle{mystyle}{
  backgroundcolor=\color{backcolour}, commentstyle=\color{codegreen},
  keywordstyle=\color{magenta},
  numberstyle=\tiny\color{codegray},
  stringstyle=\color{codepurple},
  basicstyle=\ttfamily\footnotesize,
  breakatwhitespace=false,         
  breaklines=true,                 
  captionpos=b,                    
  keepspaces=true,                 
  numbers=left,                    
  numbersep=5pt,                  
  showspaces=false,                
  showstringspaces=false,
  showtabs=false,                  
  tabsize=2
}
\usepackage{kotex}

%% file: preamble/acronyms.tex
\usepackage[acronym,nowarn,section,nogroupskip,nonumberlist]{glossaries}
\glsdisablehyper{}

\newacronym{ode}{\textsc{ode}}{Ordinary Differential Equation}
\newacronym{cnf}{\textsc{cnf}}{Continuous Normalizing Flow}
\newacronym{ivp}{\textsc{ivp}}{Initial Value Problem}
\newacronym{gan}{\textsc{gan}}{Generative Adversarial Network}
\newacronym{vae}{\textsc{vae}}{Variational Autoencoder}
\newacronym{fid}{\textsc{fid}}{Frech\'et Inception Distance}
\newacronym{nfe}{\textsc{nfe}}{Number of Function Evaluation}
\newacronym{seqrf}{\textsc{SeqRF}}{\emph{Sequential Reflow}}
\newacronym{gte}{\textsc{gte}}{\emph{global} truncation error}
\newacronym{lte}{\textsc{lte}}{\emph{local} truncation error}
\newacronym{ot}{\textsc{ot}}{Optimal Transpose}

%% file: preamble/math.tex
\newcommand{\bZ}{\mathbf{Z}}

\newcommand{\calL}{{\mathcal{L}}}

\newcommand{\calN}{{\mathcal{N}}}
\newcommand{\calO}{{\mathcal{O}}}

\newcommand{\calU}{{\mathcal{U}}}

\newcommand{\calZ}{{\mathcal{Z}}}

\newcommand{\bbE}{\mathbb{E}}

\newcommand{\bbR}{\mathbb{R}}

\theoremstyle{plain}
\newtheorem{thm}{Theorem}
\newtheorem{lem}[thm]{Lemma}
\newtheorem{prop}[thm]{Proposition}

\theoremstyle{definition}
\newtheorem{defn}{Definition}

\newcommand{\dee}{\mathrm{d}}
\newcommand{\grad}{\nabla}

\newcommand{\mathwrap}[1]{\texorpdfstring{#1}{TEXT}}

\def\[#1\]{\begin{equation}\begin{aligned}#1\end{aligned}\end{equation}}

\newcommand\norm[1]{\left\lVert#1\right\rVert}

%% file: main/00_abstract.tex
Straightening the probability flow of the continuous-time generative models, such as diffusion models or flow-based models, is the key to fast sampling through the numerical solvers, existing methods learn a linear path by directly generating the probability path the joint distribution between the noise and data distribution.
One key reason for the slow sampling speed of the ODE-based solvers that simulate these generative models is the global truncation error of the ODE solver, caused by the high curvature of the ODE trajectory, which explodes the truncation error of the numerical solvers in the low-NFE regime.
To address this challenge, We propose a novel method called \gls{seqrf}, a learning technique that straightens the probability flow to reduce the global truncation error and hence enable acceleration of sampling and improve the synthesis quality.
In both theoretical and empirical studies, we first observe the straightening property of our \gls{seqrf}.
Through empirical evaluations via \gls{seqrf} over flow-based generative models, We achieve surpassing results on CIFAR-10, CelebA-$64\times 64$, and LSUN-Church datasets.


%% file: main/01_introduction.tex
\section{Introduction}
\label{main:sec:introduction}

In recent times, continuous-time generative models, exemplified by diffusion models ~\citep{song2019ncsn, song2021scorebased, ho2020ddpm} and flow-based models~\citep{lipman2023flow,liu2023rf}, have demonstrated significant improvement across diverse generative tasks, encompassing domains such as image generation~\citep{dhariwal2021diffusion}, videos~\citep{ho2022video}, and 3D scene representation~\citep{luo2021point}, and molecular synthesis~\citep{xu2022geodiff}. 
Notably, these models have outperformed established counterparts like \glspl{gan}~\citep{goodfellow2014gan} and \glspl{vae}~\citep{kingma2013vae}.
Operating within a continuous-time framework, these models acquire proficiency in discerning the time-reversal characteristics of stochastic processes extending from the data distribution to the Gaussian noise distribution in the context of diffusion models.
Alternatively, in the case of flow-based models, they directly learn the probability flow, effectively simulating the vector field.

\begin{figure}[!t]
\centering
\includegraphics[width=0.6\linewidth]{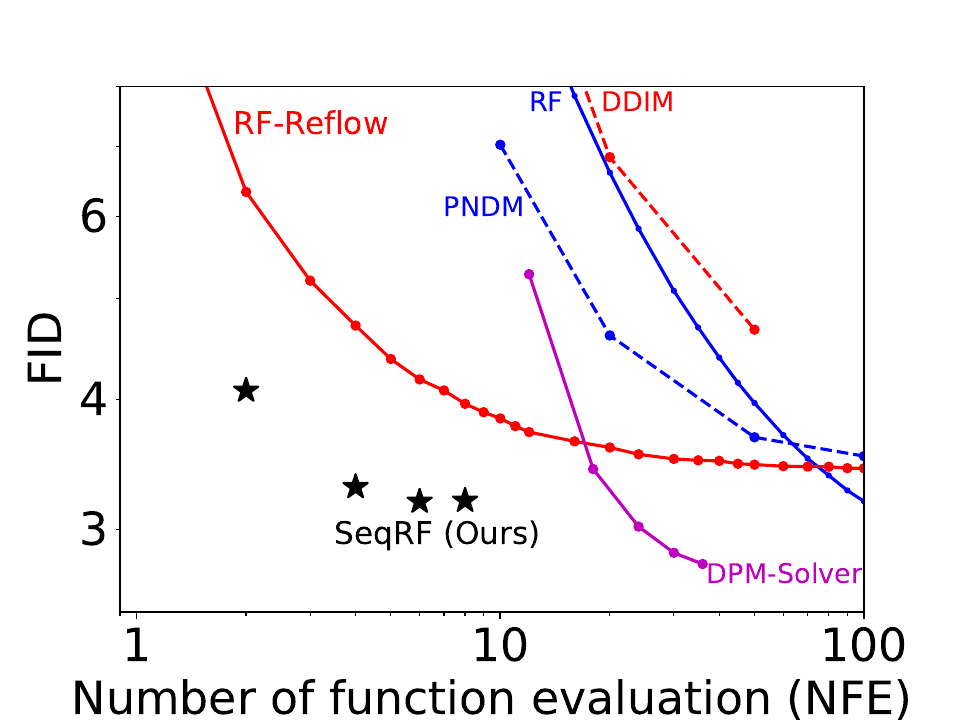} 
\caption{The overall generation result in CIFAR-10 dataset, comparing our method to existing diffusion and flow-based model solvers.
The \textbf{black starred} points stand for our proposed SeqRF method.}
\label{fig:overall_intro}
\end{figure}




\begin{figure*}[!ht]
\centering
\begin{center}
\includegraphics[width=0.8\linewidth]{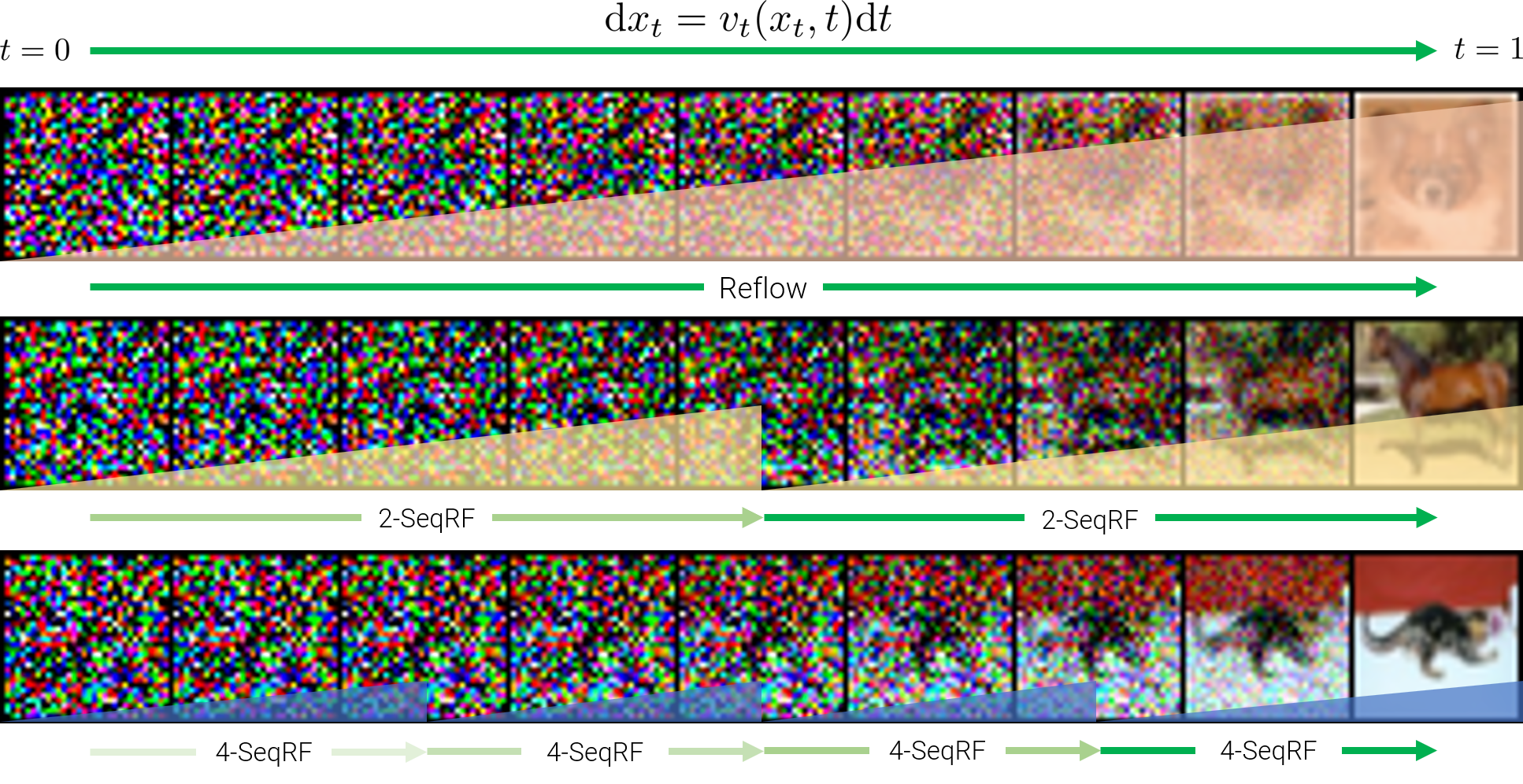}
\end{center}
\caption{The concept figure of our method.
The \red{red}, \yellow{yellow} and \blue{blue} triangles represent the truncation error being accumulated in the corresponding time.
Compared to the \red{red} reflow method, sequential reflow (SeqRF) mitigates marginal truncation error by running time-segmented ODE.
}\label{fig:concept_figure}
\end{figure*}

Recently, ~\citet{lipman2023flow} introduced a novel concept termed~\emph{flow matching} within continuous-time generative models.
This approach focuses on learning the vector field connecting the Gaussian noise and the data distribution.
The authors initially demonstrated that the marginal vector field, representing the relationship between these two distributions, is derived by marginalizing over the gradients of the conditional vector fields.
Moreover, they found that the conditional vector field yields optimal transport between the data and noise distributions, particularly when the noise distribution adheres to the standard Gaussian distribution.
Learning the (marginal) vector field involves independent sampling from both the noise and data distributions, followed by the marginalization of the conditional vector field over the data distributions.
Despite the advantageous property of identifying optimal transport paths with optimal coupling, this method faces challenges such as high learning variance and slow training speed attributable to the independent drawing of training data from data and noise distributions, which results in high gradient variance even at its convergence.
This threatens the stability of the~\gls{ode} solver due to the accumulation of truncation errors.
In response to this challenge,~\citet{liu2023rf} proposed a method to straighten the trajectory by leveraging the joint distribution.
This involves running the numerical ODE solver from the noise to approximate the data point.
However, traversing a long path from noise to image space still leaves exploding global truncation errors unaddressed.


To address the challenge posed by truncation errors, we introduce a novel framework termed~\acrfull{seqrf}.
This novel approach represents a straightforward and effective training technique for flow-based generative models, specifically designed to alleviate the truncation error issue.
The key innovation lies in the segmentation of the~\gls{ode} trajectory with respect to the time domain.
In this strategy, we harness the joint distribution by partially traversing the solver, as opposed to acquiring the complete data with the entire trajectory.
The overall concept of our method is briefly introduced in~\cref{fig:concept_figure}.

Our contribution is summarized as follows.
\paragraph{Controlling the global truncation error}
We begin by highlighting the observation that the global truncation error of numerical \gls{ode} solvers experiences explosive growth, exhibiting superlinear escalation.
This underscores the significance of generating the joint distribution from segmented time domains, thereby ensuring diminished global truncation errors through the strategic halting of the solver before the error reaches critical levels.

\paragraph{Main Proposal: Sequential Reflow for Flow Matching}
Our primary contribution is the introduction of~\emph{sequential reflow} as a method for flow matching in generative modeling.
This flow straightening technique aims to diminish the curvature of segmented time schedules.
It achieves this by generating the joint distribution between data points at different time steps.
Specifically, the source data point, originating from the training set enveloped in noise, undergoes the ODE solver constructed by pre-trained continuous-time generative models such as flow-based or diffusion models.
This marks a departure from conventional reflow methods, which straighten the entire ODE trajectory.

\paragraph{Validation of Sequential Reflow}
We empirically validate that our sequential reflow method accelerates the sampling process, thereby enhancing the image synthesis quality; the sampling quality improves with the rectified flow, achieved by employing the sequential reflow method subsequent to the initial flow matching procedure.
Furthermore, We implement distillation on each time fragment, enabling the traversal of a single time segment at a single function evaluation. This strategic approach results in superior performance, achieving a remarkable 3.19~\gls{fid} with only 6 function evaluations on the CIFAR-10 dataset, as sketched in~\cref{fig:overall}.

%% file: main/02_backgrounds.tex
\section{Background: Continuous Normalizing Flow and Flow Matching}\label{sec:background}
Consider the problem of constructing the probability path between the two distributions $\pi_0$ and $\pi_1$ in the data space $\bbR^D$.
Let the~\emph{time-dependent} vector field be $u:\bbR^D\times[0,1]\to\bbR^D$ with $t\in[0,1]$.
Then the~\gls{ivp} on an~\gls{ode} generated by this vector field $u$ is given as
\[
\dee x_t &= u(x_t, t) \dee t.
\label{eq:ode}
\]
\citet{chen2018node,grathwohl2019ffjord} suggested the generative model called~\emph{\gls{cnf}} that reshapes the simple, easy-to-sample density (i.e., Gaussian noise $p_0$) into the data distribution $p_1$, by constructing the vector field $u$ with a neural network $v_\theta(x, t)$, parameterized by $v_\theta:(\bbR^D,\bbR)\to\bbR^D$.
This vector field $v_\theta$ is used to generate a time-dependent continuously differentiable map called~\emph{flow} $\phi_t:\bbR^D\to\bbR^D$ if the random variable generated with $X_t\sim\phi_t$ satisfy~\cref{eq:ode}.
A vector field $v_\theta$ is called to generate the flow $\phi_t$ if it satisfies the push-forward (i.e. (generalized) change of variables) equation
\[
p_t
&=
\left[ \phi_t \right]_* p_0,\\
\text{where }\left[ \phi_t \right]_* p_0 (x)
&=
p_0 (\phi_t^{-1} (x)) \left| \det \left[ \frac{\partial \phi_t^{-1}}{\partial x}(x) \right] \right|.
\]
However, constructing the~\gls{cnf} requires backpropagation of the entire adjoint equation with the NLL objective, which requires full forward simulation and gradient estimation of the vector field over the time domain $[0,1]$.
The~\emph{flow matching}~\citep{lipman2023flow,liu2023rf} algorithm overcomes this limit with a simulation-free (e.g., does not require a complete forward pass for training) algorithm by replacing this with the $\ell_2$-square objective
\[
\calL_\text{FM}(\theta)
&=
\bbE_{t,x_1}
\left[
\norm{v_\theta(x,t) - u(x_t, t)}_2^2
\right],\label{eq:fm_objective}
\]
where $u$ is true vector field defined in~\eqref{eq:ode}.
However, the computational intractability of $v_\theta$ makes it difficult to directly learn from the true objective~\cref{eq:fm_objective}.
\citet{lipman2023flow} found that by using the conditional flow matching objective instead,
\[
\calL_\text{CFM}(\theta)
&=
\bbE_{t,x_1,x_t|x_1}
\left[
\norm{v_\theta(x_t,t) - u(x_t, t|x_1)}_2^2
\right],\label{eq:cfm_objective}
\]
where $p_t(x_t|x_1)$ is the conditional probability density of noisy data $x_t$ given the data $x_1$, we can learn the flow matching model with conditional flow matching objective based on the proposition below.

\begin{prop}[Equivalence of the FM and CFM objective]
The gradient of~\cref{eq:fm_objective} and~\cref{eq:cfm_objective} is equal.
That is, $\calL_\mathrm{FM}(\theta)=\calL_\mathrm{CFM}(\theta) + C$ for a constant $C$. The detailed proof is described in~\cref{app:fm_cfm}.
\end{prop}

\subsection{Constructing Straight Vector Field with Flow matching}\label{sec:fm_ot}

\citet{lipman2023flow} proposed a natural choice of constructing flow between two distributions by taking
\[
p_t\sim\calN(tx_1, (1-(1-\sigma_\mathrm{min})t)^2 I)
\label{eq:ot_interpolant}
\]
at time $t\in[0,1]$.
This setting is directly related to constructing the straightened flow between two distributions.
\citet{mccan1997ot} showed that with two Gaussian distributions at the endpoints $p_0\sim\calN(0,I)$ and $p_1\sim\calN(x_1,\sigma_\mathrm{min}^2 I)$, the vector field 
\[
u(x_t,t|x_1,\sigma_\mathrm{min})
&=
\frac{x_t - tx_1}{(1-\sigma_\mathrm{min})(1-t)+\sigma_\mathrm{min}}
\label{eq:conditional_ot_interpolant}
\]
achieves the conditionally straight path between $p_0$ and $p_1$.

\citet{liu2023rf, albergo2023stochastic} designed the source distribution $p_1$ and the conditional distribution by approaching $\sigma_\text{min}\to 0$, as
\[
p_1\sim \pi_0 \times \pi_1, \quad u(x_t,t|x_1)=\lim_{\sigma_\mathrm{min}\to 0 } u(x_t,t|x_1,\sigma_\mathrm{min}).
\]
Then, the CFM objective becomes
\[
\bbE_{t,x_0,x_1,n}
\left[
\norm{v_\theta(x_t,t ) - (tx_1 + (1-t)x_0 + n)}\label{eq:interpolant}
\right],
\]
where $(x_0,x_1)\sim\pi_0 \times\pi_1$ and $n\sim\calN(0,I)$,
which is reduced to
\[
\bbE_{t,x_0,x_1}
\left[
\norm{v_\theta(x_t,t) - (tx_1 + (1-t)x_0)}\label{eq:rf}
\right]
\]
which corresponds to the rectified flow objective~\citep{liu2023rf,liu2023instaflow}.


\subsection{The Truncation Errors of the Numerical Solvers and Rectified Flow}
Let the~\gls{ivp} of an~\gls{ode} be given as follows:
\[
\begin{cases}
\dee x_t = v_\theta(x_t, t)\dee t,\\
x_a=x(a)
\end{cases}\quad t\in[a,b],
\label{eq:ode_ivp}
\]
where $v_\theta$ is the (learned) neural network, $x(a)$ is the initial value, and $(a,b)$ are the initial and terminal points of the~\gls{ode}, respectively.
In general, a solver with the~\gls{ode}~\cref{eq:ode_ivp} is defined by generating the sequence $\{x_a=x_{t_0},x_{t_i},\cdots,x_{t_N}=x_b\}$ with the following equation
\[
x_{t_{i+1}} = x_{t_i} +(t_{i+1} - t_i) A(x_{t_i}, t_i, h_i; v_\theta)
\label{eq:general_solver}
\]
where $A(x_{t_i},t_i;v_\theta)$ is the increment function and $h_i=t_{i+1}-t_i$ is the interval between two consecutive steps.
Then the \emph{truncation error} of a solver is defined by the difference between the true solution of~\eqref{eq:ode_ivp} and the approximation from~\eqref{eq:general_solver}.
There are two kinds of truncation errors: the \emph{local} and \emph{global} truncation errors.
\begin{defn}[truncation errors]
Let $x_{t_i}$ be an estimate with~\eqref{eq:general_solver} of \gls{ode}~\eqref{eq:ode_ivp} at time $t_i$.
Then the~\gls{lte} $\tau(t_i, h_i)$ is
\[
\tau(t_i,h_i) 
&=
(\hat{x}_{t_i+h_i}(x_{t_i}) - x_{t_i}) - A(x_{t_i}, t_i, h_i; v_\theta).
\]
where $\hat{x}_{t_i+h_i}(x_{t_i})$ is the true solution of~\eqref{eq:ode_ivp} which starts at point $x_{t_i}$ from time $t_i$ to $t_i+h_i$.
The~\gls{gte} $E(t_i)$ at time $t$ is the accumulation of~\gls{lte} over time $\{t_j\}_{j=1}^i$,
\[
E(t_i)
&=
(x(t_i)-x(a))-\sum_{j=1}^i A(x_{t_j}, t_j, h_j; v_\theta).
\]
where $x(t)$ is the true solution of~\eqref{eq:ode_ivp} at time $t$.
\end{defn}
As the CFM objective~\eqref{eq:cfm_objective} with flow matching or rectified flow achieves conditional OT mapping between two conditional distributions, it is obvious that with the same increment function $A(x_{t_i}, t_i; v_\theta)$ of the numerical solver~\eqref{eq:general_solver} from a point $x_a$ in the source distribution $a$ to the conditional target distribution $p_b(\cdot|x_a)$ with one-step linear solver with timestep interval $h=b-a$, eliciting zero global truncation error.

Despite that this provably holds for optimal transport mapping between conditional distributions, this generally does not hold for OT mapping between marginal distributions.
\cref{fig:gte} depicts the global truncation error between the approximately true~\gls{ode} solver and the small-step solver learned by the flow matching algorithm, comparing the~\gls{gte} between the few-step Euler sampler to the 480-step Euler sampler as baseline.
The \textbf{bold black} line, which directly trains the flow matching model via the CFM objective~\eqref{eq:cfm_objective}, has high~\gls{gte}.
The reflow method~\citep{liu2023rf}, marked in a \textbf{dashed black} line, generates pairs of noise and data points from the pre-trained flow matching models and takes each pair as the new joint distribution.

In the next section, We introduce~\acrfull{seqrf} method, displayed by \textbf{\blue{blue}} and \textbf{\red{red}} lines in \cref{fig:gte}, which successfully suppress the~\gls{gte} of the~\gls{ode} solver and hence achieve high image generation quality with a small~\gls{nfe}.

%% file: main/03_methods.tex
\section{Main method}
Several factors cause the numerical solver to have a high~\gls{gte}, implying that this solver is ill-posed.
In~\cref{sec:gte}, we investigate the aspects that affect to~\gls{gte} by deriving the upper bounds of~\gls{gte}.
To summarize, the (1) \gls{lte} of the solver in a single step, (2) the variance of the drift $v_\theta$, and (3) the total interval of the numerical solver.
In~\cref{sec:seqrf}, we propose our new method, \acrfull{seqrf}, that mitigates these errors to construct a fast and efficient numerical solver that successfully simulates an~\gls{ode}.

\subsection{On the Global Truncation Error of ODE Solver}\label{sec:gte}
In this section, we delve into the factors that cause high global truncation error of the ODE solver to explode.
The following~\cref{thm:gte_bound} implies that the truncation error of an ODE solver depends on both the \emph{local truncation error} of the solver and the \emph{Lipschitz constant} of $v_\theta$.
Following this, \cref{thm:truncation} shows that this is also affected by the \emph{total interval length} of the~\gls{ode}.


\begin{figure}[!t]
\centering
\begin{center}
\includegraphics[width=0.6\linewidth]{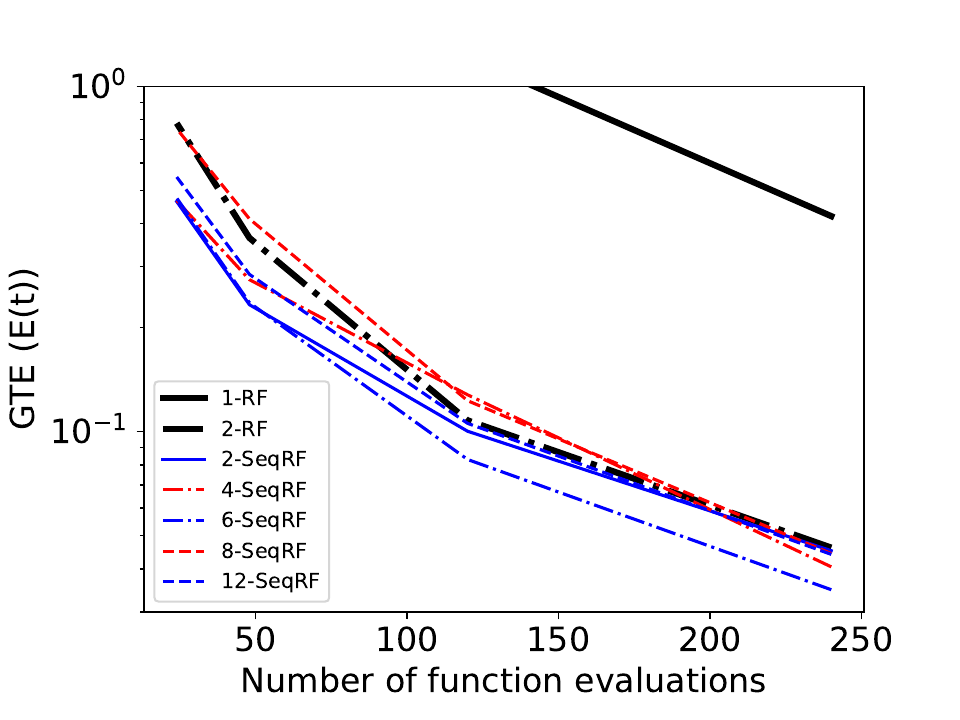}
\end{center}
\caption{The global truncation error over~\gls{nfe} in CIFAR-10 dataset, compared to the oracle Euler-480 step solver.
Our \textbf{SeqRF} methods, shown in \textbf{\blue{blue}} and \textbf{\red{red}} lines, deploy lower global truncation errors.}\label{fig:gte}
\end{figure}

\begin{thm}[Upper Bound of GTE w.r.t. LTE]\label{thm:gte_bound}
Let the local truncation error of the ODE solver be $\tau(t)$, and $M_\textrm{sup}(t)$ be the maximum Lipschitz constant of $f$ at time $t$.
Then The global truncation error $E(c)$ at time $c$ of the one-step ODE solver is bounded by
\[
E(c)
&\leq
\int_{t=a}^c \tau(t)\exp\left(\int_{t'=t}^c M_\mathrm{sup}(t')\dee t'\right) \dee t
\label{eq:gte_wrt_lte}
\]
for all $c\in[a,b]$.
\end{thm}

\begin{proof}
Please refer to~\cref{app:proof:gte_bound}.
\end{proof}

\begin{figure}[!t]
\centering
\includegraphics[width=\linewidth]{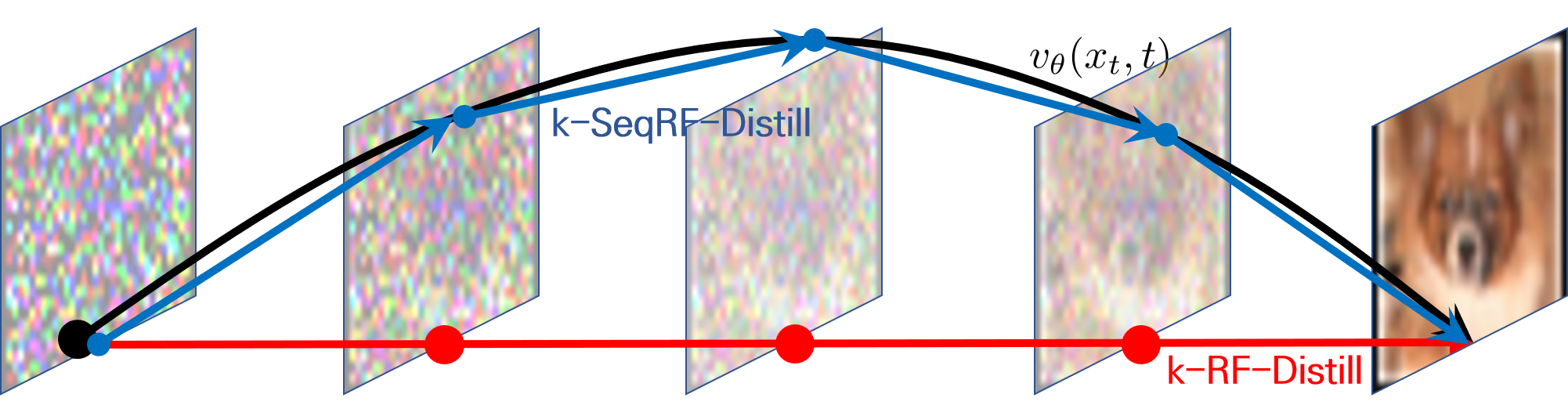} 
\caption{The concept figure of $k$-SeqRF-Distill, contrast to the RF-Distill method~\citep{liu2023rf}.}
\label{fig:concept_seqrf_distill}
\end{figure}

Then, \citet{dahlquist1963} provided the upper bound of the global truncation error of the generalized~\gls{ode} solvers as follows.

\begin{lem}[Dahlquist Equivalence Theorem ~\citep{dahlquist1963}]
Let $[a,b]$ be divided by $h$ equidistributed intervals, i.e., $t_i=a+\frac{i}{h}(b-a)$.
Then the generalized linear multistep method is defined by
\[
&x_{a + (i+1)h}
=
\\&\sum_{j=0}^p \alpha_j x_{a-jh} + h \sum_{j=-1}^p \beta_j v(x_{a-jh},a-jh)),\quad i\geq p.
\label{eq:linear_multistep}
\]
Then for a linear multistep method of order $p$ that is consistent with the~\gls{ode} (\ref{eq:ode_ivp}) where $v_t$ satisfy the Lipschitz condition and with fixed initial value $x_a$ at $t=a$, the global truncation error is $\calO(h^p)$.
\label{lem:dahlquist}
\end{lem}
With~\cref{lem:dahlquist}, the global truncation error accumulated by the ODE solver with the same complexity is reduced when the total length of the time interval becomes shorter.
To be precise, the global truncation error is reduced by the order of $K^{p-1}$, where $K$ is the number of intervals and $r$ is the order of the ODE solver, according to the following~\cref{thm:truncation}.
\begin{thm}[Global Truncation Error with Increasing Time Segments]
Suppose that a linear multistep method (\ref{eq:linear_multistep}) successfully simulates the~\gls{ode} (\ref{eq:ode_ivp}).
And suppose that the~\gls{nfe} of linear multistep method from time $(a, t)$ is given the same as $\frac{b-a}{K}$. 
Let $K$ be the number of equidistributed intervals that segment $[a, b]$.
And let (\ref{eq:linear_multistep}) have of order $p$, then  the local truncation error is $\calO(h^{p+1})$, then the following holds.
\begin{enumerate}
\item 
The order of the local truncation error is $\tau(t, h) = \calO(h^{p+1})$.
\item 
The global truncation error at time $t$ is $E(t)=\calO\left(K^{1-p}\right)E(c)$
\end{enumerate}
\label{thm:truncation}
\end{thm}
\begin{proof}
Please refer to \cref{app:dahlquist}.
\end{proof}

\begin{figure}[!t]
\centering
\includegraphics[width=0.6\linewidth]{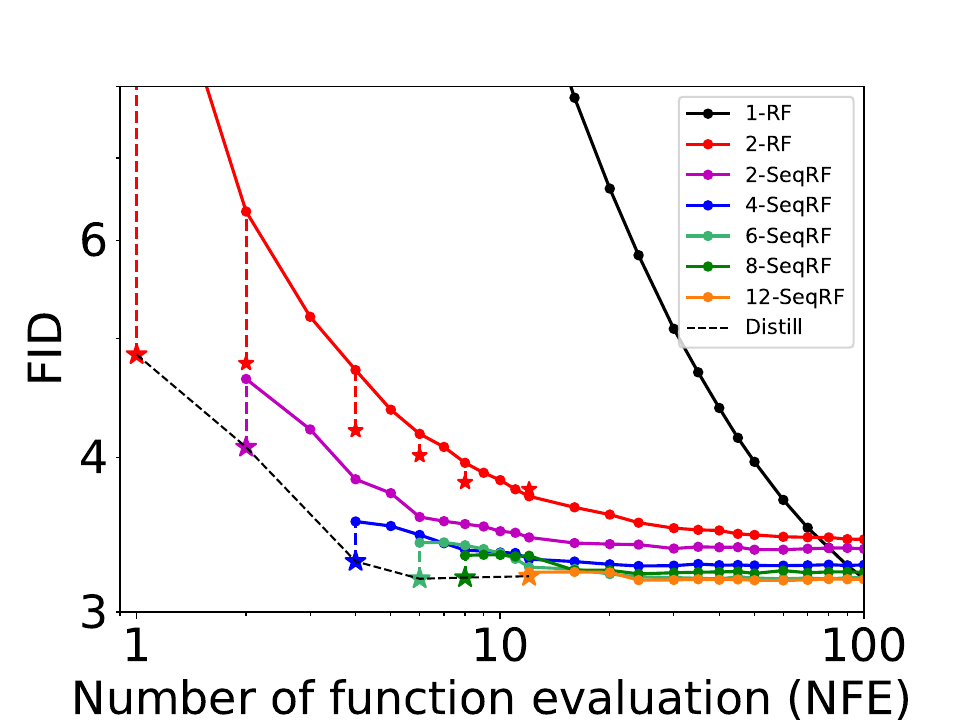} 
\caption{Generation performance of Sequential reflow, compared to the original rectified flow method.
The \textbf{\color{black}black} and \textbf{\color{red}red}, and \textbf{\blue{oth}\green{er}} lines represent the rectified flow (1-RF), the reflowed model (2-RF), and the $k$-SeqRF ($k=\{1, 2, 4, 6, 8, 12\}$ models, respectively. 
The starred points denote the performance of \emph{distilled} models.}
\label{fig:overall}
\end{figure}

\begin{figure}[!ht]
\centering
\includegraphics[width=0.49\linewidth]{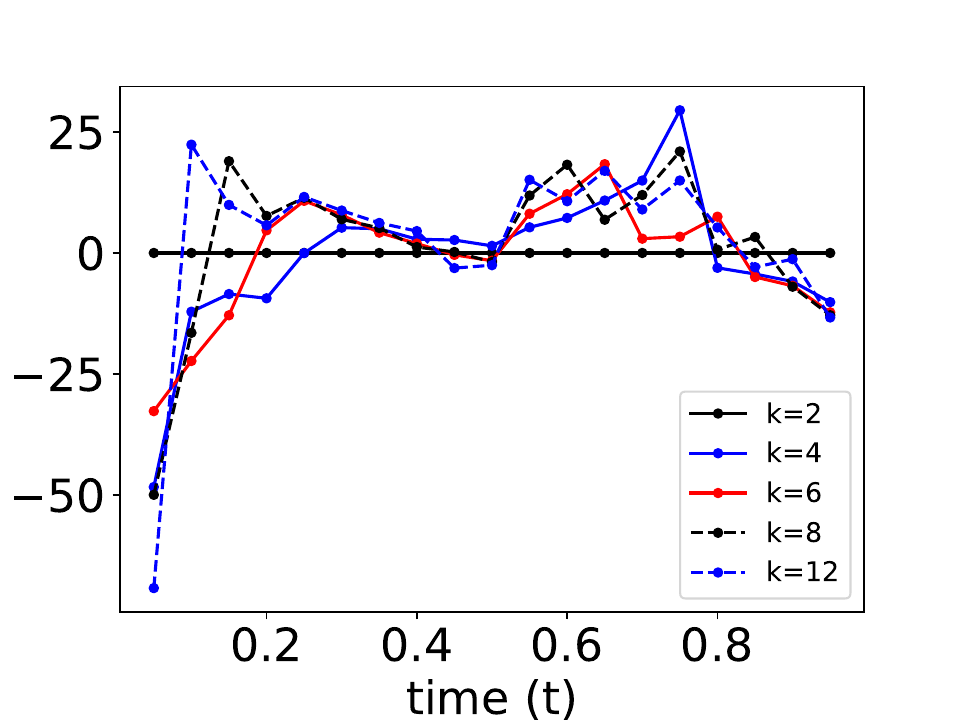} 
\includegraphics[width=0.49\linewidth]{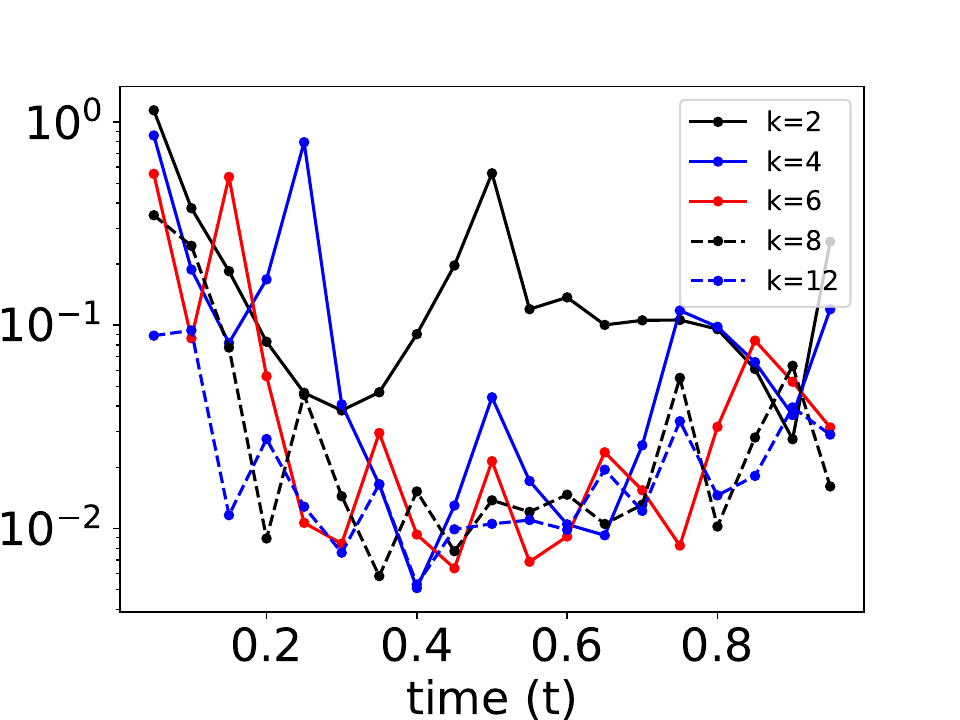} 
\caption{Empirical results on the (a) average Lipschitzness over $x_t$ and (b) \emph{$\norm{v_\theta(x_t, t)}_2^2$} over different $K$ in SeqRF.
In both two cases, the black line (2-SeqRF) with minimal segments has maximum values.}
\label{fig:ablation}
\end{figure}

\begin{figure*}[!ht]
\centering
\begin{tabular}{ccc}
\begin{subfigure}{0.32\textwidth}
\centering
\includegraphics[width=\linewidth]{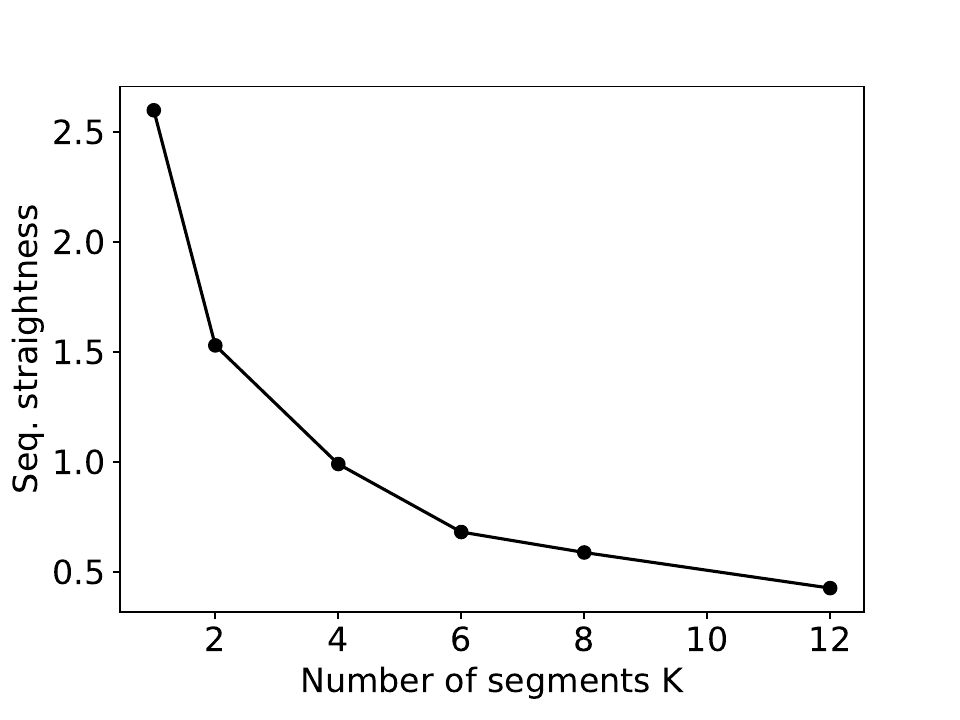}
\caption{Over the number of segments $K$.}
\label{fig:seq_str_k}
\end{subfigure}
&
\begin{subfigure}{0.32\textwidth}
\centering
\includegraphics[width=\linewidth]{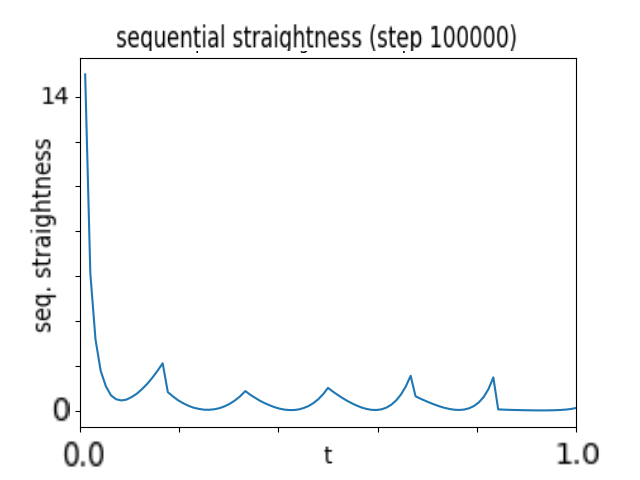}
\caption{Seq. straightness with $k=6$ over $t$.}
\label{fig:seq_str_t6}
\end{subfigure}
&
\begin{subfigure}{0.32\textwidth}
\centering
\includegraphics[width=\linewidth]{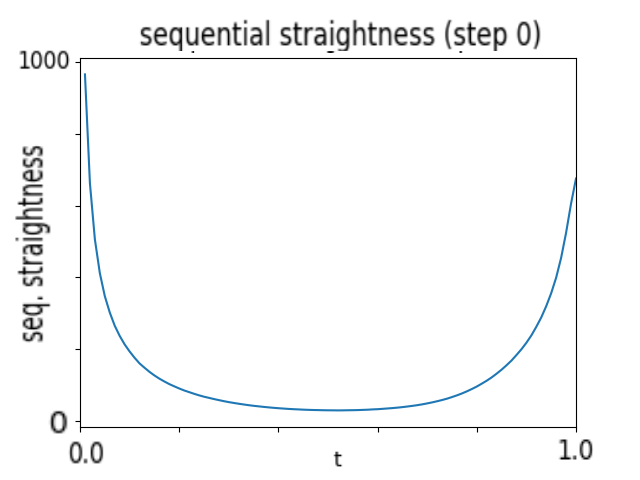}
\caption{Straightness of 1-RF over $t$.}
\label{fig:seq_str_tbaseline}
\end{subfigure}
\end{tabular}
\caption{
Sequential straightness result for CIFAR-10 dataset, with 100-step Euler solver.
$t=0$ and $t=1$ stand for near-data and near-noise regimes, respectively.
Straightness results for other $k$'s are introduced in~\cref{app:sec:seq_str_plot_t}.}\label{fig:seq_straightness}
\end{figure*}
\subsection{Sequential Reflow: Flow Straightening by Suppressing the Truncation Error}
\label{sec:seqrf}
In this section, we propose our new method, ~\emph{sequential reflow}, that improves the ODE-based generative models by suppressing the truncation error of the numerical solver, by refining the pre-trained~\gls{ode}~\citep{liu2023rf,liu2023instaflow,albergo2023stochastic,albergo2023interpolant}.
\cref{alg:train_srf} summarizes our~\gls{seqrf} method.

\paragraph{Generating Joint Distribution with Time Segmentation of the ODE Trajectory}
Let the \gls{ode}-\gls{ivp} problem be defined by~\cref{eq:ode_ivp}.
Then the entire time interval of this~\gls{ode} is given as $[a, b]$.
In~\cref{thm:truncation} of~\cref{sec:gte}, we found that the global truncation error of the solver is inversely proportional to the $(p-1)^\mathrm{th}$ power of the number of time segments $K$ for the solver having order $p$.
Inspired by this, we first divide the~\gls{ode} trajectory by $K$ segments, $\{a=t_0,t_1,\cdots, t_K=b\}$, to generate the joint distribution, where the~\emph{source} data point at time $t_k$ is sampled from the OT interpolant $p_{t_k}$ from~\cref{eq:ot_interpolant} with $\sigma_\mathrm{min}\to0$.
Then the joint distribution generated by our $K$-SeqRF method is as follows:
\[
(x_{t_k}, \hat{x}_{t_{k+1}}(x_{t_k}))
\sim
p_{t_k}(\cdot) \times \texttt{ODE-Solver}(x_{t_k},t_k, t_{k+1};v_\theta)
\label{eq:joint_seqrf}
\]
for $k\in[0:K-1]$, where \texttt{ODE-Solver} denotes the numerical solver on use, which can be any linear multistep method defined by~\cref{eq:linear_multistep}, such as Euler, Heun, or Runge-Kutta methods.

Our SeqRF differs from the Reflow method~\citep{liu2023rf,liu2023instaflow}, which also attempted to straighten the flow-based models by constructing joint distribution.
The Reflow method generates the joint distribution
\[(x_a, \hat{x}_b)\sim\calN(0,I)\times \texttt{ODE-Solver}(x_a,a,b;v_\theta),
\label{eq:joint_rf}\]
by solving~\gls{ode} from $a$ to $b$.
In particular, the reflow method is the special case of our method with $K=1$.
We visualized how SeqRF differs from the na\"ive reflow method in~\cref{fig:concept_seqrf_distill}.

\paragraph{Training with Joint Distribution}
Let the pair of data points drawn from the joint distribution with~\cref{eq:joint_seqrf} be $\{\pi(x_{t_K}, \hat{x}_{t_{K+1}}(x_{t_K}))\}_{i=0}^{K-1}$.
\[
\calL_\mathrm{seq}
&=
\bbE_{x_s,k}\left[\norm{v_\theta(x_s, s) - \frac{x_s - x_{t_{k-1}}}{t_i - t_{k-1}}}_2^2\right]
\label{eq:seqrf_obj}
\]

After training the probability path with joint distribution, we may fine-tune our model with distillation to move directly toward the whole time segment.
That is, the number of function evaluations is the same as the number of time segments after distillation finishes.
The distillation process is done as follows:
\begin{itemize}
\item
we use the same generated pairs of data points as in~\cref{eq:joint_seqrf} as the joint distribution, with the same objective~\cref{eq:seqrf_obj}.
\item 
Different from the primary SeqRF training process, We fix the time to the starting point of $t_k$ instead of uniformly training from $t\in (t_k, t_{k+1})$.
\end{itemize}
Then, by distilling the SeqRF model, we directly take advantage of the straightened probability flow of the previously trained SeqRF model.

\paragraph{Variance Reduction with Sequential Reflow}
\citet{pooladian2023multisample} provides the variance reduction argument that sheds light on using joint distribution as training data, as follows.
\begin{prop}[\citet{pooladian2023multisample}, Lemma 3.2]
Let $(x_a, x_b)$ be jointly drawn from the joint density $\pi=\pi_a\times\pi_b$.
And let the flow-matching objective from joint distribution be given as
\[
\calL_\mathrm{JCFM}
&=
\bbE_{x_a,x_b}\left[\norm{v_\theta(x_t, t)-u_t(x_t|x_1)}_2^2\right].
\]
Then the expectation of the total variance of the gradient at a fixed $(x,t)$ is bounded as
\begin{align}
&\bbE_{t,x}\left[\mathrm{Tr}(\mathrm{Cov}_{p_t(x_1|x)}
\left(
\grad_\theta \norm{v_\theta(x, t)-u_t(x|x_a)}_2^2
\right)
\right]\nonumber\\
&\quad \leq \max_{t,x}\norm{\grad_\theta v_\theta(x, t)}_2^2 \calL_\mathrm{JCFM}.
\label{eq:pooladian_bound}
\end{align}
\label{prop:joint}
\end{prop}
This proposition implies that training from joint distributions, instead of the product of independent distributions, improves training stability and thus efficiently constructs the optimal transport (OT) mapping between two distributions.

\paragraph{Empirical Validation on Bounds}
$M_\mathrm{sup}(t)$ \eqref{eq:gte_wrt_lte} and $\norm{\grad_\theta v_\theta(x,t)}_2^2$ \eqref{eq:pooladian_bound} are key aspects for upper bounds of~\gls{gte}.
To show that SeqRF successfully mitigates the~\gls{gte} and obtain more exact solution of the solver, we empirically showed that $M_\mathrm{sup}(t)$ and $\norm{\grad_\theta v_\theta(x,t)}_2^2$ tend to decrease as $k$ increases.

\paragraph{Flow Straightening Effect with Sequential Reflow}
To validate how much the vector field approximates a single-step solver, \citet{liu2023rf} proposed a measure called~\emph{straightness} of a continuously differentiable process $\bZ=\{Z_t\}_{t=0}^1$ that defines the ODE~\eqref{eq:ode_ivp}, defined by
\[
S(\bZ)
&=
\int_0^1 \norm{(Z_1 - Z_0) - \frac{\dee}{\dee t}Z_t}^2 \dee t.
\]
$S(\bZ)$ is zero if and only if $v(x_t)$ is locally Lipschitz with zero dilation $M(t)$ almost surely.
Hence, this implies that having low straightness represents for low truncation error.

To generalize this to the stiff~\gls{ode} formulated by~\gls{seqrf}, we propose a metric named \emph{sequential straightness},
\[
S_\text{seq}(\calZ)
&=
\sum_{i=1}^K \int_{t_{i-1}}^{t_i} \norm{\frac{Z_{t_i}^i - Z_{t_{i-1}}^i}{t_i - t_{i-1}} - \frac{\dee}{\dee t} Z_t^i}^2 \dee t,
\]
where $\calZ = \{\bZ^i\}_{i=1}^K$ is a collection of the time-segmented processes.
Like $S(\bZ)$, zero $S_\text{seq}(\calZ)$ denotes that each time-segmented process is completely straight, and this intends that the sequential straightness is zero if and only if the dilation $M(t)$ is zero almost surely for all $t\in(a,b)$.

\cref{fig:seq_straightness} demonstrates the sequential straightness of the rectified flow models trained on CIFAR-10 datasets for the number of segments, where the number of reflow and distillation pairs $1$M.
\cref{fig:seq_str_k} demonstrates that the sequential straightness lowers in both datasets as we train with finer intervals, implying that we achieved a superior straightening performance by using the SeqRF method.
For instance, \cref{fig:seq_str_t6} displays the straightness with respect to time at $6$-SeqRF, from near-noise regime at time $0$ to near-data at time $1$.

%% file: main/04_related_works.tex
\section{Related work}\label{sec:related}

We introduce the most relevant works to our paper in the main section.
For additional related works, please refer to~\cref{app:sec:more_related_works}.

\paragraph{Flow Matching}
Recently, straightening the continuous flow by mitigating the curvature of the probability path has been considered by considering the optimal transport regularization by introducing~\gls{cnf} norms~\citep{finlay2020node,onken2021otflow,bhaskar2022curvature} for training neural ODE, and~\citet{kelly2020learning} learned the straightened neural ODEs by learning to decrease the surrogate higher-order derivatives.
\citet{lee2023minimizing} showed that minimizing the curvature in flow matching is equivalent to the $\beta$-VAE in variational inference.
As a generalization of finding optimal transport interpolant, \citet{albergo2023interpolant,albergo2023stochastic} proposed stochastic interpolant, which unifies the flows and diffusions by showing that the probability density of the noisy interpolant between two distributions satisfies the Fokker-Planck equations of an existing diffusion model.

\paragraph{Variance Reduction with Optimal Coupling}
Our work can also be interpreted to match the noise and data distribution pairs to make training more efficient by reducing training variance.
\citet{pooladian2023multisample} used the minibatch coupling to generate joint distributions in a simulation-free manner by constructing a doubly stochastic matrix with transition probabilities and obtaining coupling from this matrix.
Even though this generates one-to-one matching between noise and images, this is numerically intractable when the number of minibatch increases.
\citet{tong2023improving} also introduced the concept of minibatch optimal transport by finding optimal coupling within a given minibatch.
\citet{tong2023simulation} further expanded this approach to generalized flow (e.g., Schr\"odinger bridge) in terms of stochastic dynamics.

%% file: main/05_experiments.tex
\section{Experiments}\label{sec:experiment}

\input{tables/c10_main}
\input{tables/celeba_main}

In this section, we take empirical evaluations of SeqRF, compared to other generative models, especially with diffusion and flow-based models in CIFAR-10, CelebA and LSUN-Church datasets.
We used the \texttt{NCSN++} architecture based on the \texttt{score\_sde} \footnote{\texttt{\href{https://github.com/yang-song/score\_sde}{github.com/yang-song/score\_sde}}} repository.
Each model is trained with \texttt{JAX/Flax} packages on TPU-v2-8 (CIFAR-10) and TPU-v3-8 (CelebA, LSUN-Church) nodes.
In CIFAR-10 and CelebA datasets, we first generated 1M and 300k reflow pairs to train our SeqRF models and for distillation. And for LSUN-Church dataset, we instead used 100k reflow pairs.
We did an ablation on the number of reflow pairs in \cref{app:sec:n_reflow}.
In addition, the whole experimental details including the architectural design and hyperparameters are described in~\cref{app:experiment}.

\subsection{Image Generation Result}
We demonstrate our image synthesis quality in~\cref{tab:c10_main} in CIFAR-10, compared to existing flow matching methods including rectified flow~\citep{liu2023rf}, conditional flow matching~\citep{lipman2023flow}, I-CFM and OT-CFM~\citep{tong2023improving}.
For the baseline rectified flow, we imported the pre-trained model from the \texttt{RectifiedFlow}\footnote{\texttt{\href{https://github.com/gnobitab/RectifiedFlow}{github.com/gnobitab/RectifiedFlow}}} repository, and re-implemented the method using \texttt{JAX/Flax} packages.
For a fair comparison, we report the \texttt{JAX/Flax} results in our experiments.
\footnote{The baseline FID is $0.01-0.02$ higher than reported result measured with \texttt{PyTorch}-based code.}
Our generation result surpasses existing flow-based methods and diffusion models, achieving $3.191$ FID with $6$ steps with distillation.
This not only surpasses existing diffusion model approaches, but also the existing distillation methods for flow-based models.

We present the image generation results for CelebA and LSUN-Church datasets in \cref{tab:others_main}.
\cref{fig:cifar_images} displays some non-curated CIFAR-10 images generated from the same random seed with $k$-SeqRF models.
Please refer to~\cref{app:sec:images} for further image examples such as LSUN-Church and CelebA datasets.

%% file: tables/c10_main.tex
\begin{table}[!ht]
  \small
  \centering
  \caption{
  CIFAR-10 image generation result. 
  For $\{$NFE, IS, KID$\}$/$\{$IS$\}$, lower and higher value represents better result, respectively.
  The RK45-$\{\mathrm{tol}\}$ solver denotes the Runge-Kutta solver of order 4 with absolute and relative tolerance $\mathrm{tol}$.
  The KID measure is divided by $10^4$.
  For $k$-SeqRF method, we choose the NFE where the FID converges.
  The full FID result over NFEs are demonstrated in~\cref{fig:overall}.
  We compared ours to DDIM \citep{song2021ddim}, DPM-Solver \citep{lu2022dpmsolver} and PNDM \citep{liu2022pndm}.}
  \begin{tabular}{lrrrrl}
  \toprule
  \multicolumn{6}{l}{\textbf{Flow-based methods}} \\
  \midrule
  Method & FID & IS & KID  & NFE  & Solver \\
  \midrule
  1-RF &62.173 & 9.44 & 64.70& 20  & RK45-$10^{-2}$ \\
       & 3.176 & 9.81 & 9.63 & 50  & RK45-$10^{-3}$ \\
       & 2.607 & 9.60 & 7.62 & 100 & RK45-$10^{-4}$ \\
       & 2.593 & 9.60 & 7.49 & 254 & RK45-$10^{-5}$ \\
  \midrule
  \multicolumn{6}{l}{1-RF-Distill~\citep{liu2023rf}} \\
  \hspace{2mm}$k=1$
       & 4.858 & 9.08 & 17.56 & 1 &   \\
  \hspace{2mm}$k=2$
       & 4.710 & 9.06 & 23.72 & 2 & - \\
  \hspace{2mm}$k=4$
       & 4.210 & 9.16 & 19.66 & 4 & - \\
  \hspace{2mm}$k=6$
       & 4.018 & 9.14 & 18.28 & 6 & - \\
  \hspace{2mm}$k=8$
       & 3.822 & 9.22 & 15.35 & 8 & - \\
  \hspace{2mm}$k=12$
       & 3.772 & 9.24 & 15.65 & 12 & - \\
  \midrule
  2-RF & 28.077& 9.27 & 304.1 & 21 & RK45-$10^{-2}$ \\
       & 3.358 & 9.27 & 107.9 & 50 & RK45-$10^{-3}$ \\
  \midrule
OT-CFM & 20.86 &    - &     - & 10 & Euler \\
 I-CFM & 10.68 &    - &     - & 50 & Euler \\
  \midrule
  \multicolumn{6}{l}{\textbf{$k$-SeqRF} (Ours)} \\
  \hspace{2mm}$k=2$
       & 3.377 & 9.43 & 12.04 & 30 & Euler \\
  \hspace{2mm}$k=4$
       & 3.269 & 9.44 & 11.34 & 24 & Euler \\
  \hspace{2mm}$k=6$
       & 3.191 & 9.46 & 11.70 & 40 & Euler \\
  \hspace{2mm}$k=8$
       & 3.221 & 9.47 & 11.46 & 24 & Euler \\
  \hspace{2mm}$k=12$
       & \textbf{3.184} & \textbf{9.53} & 11.25 & 24 & Euler \\
  \midrule
  \multicolumn{6}{l}{\textbf{$k$-SeqRF-Distill} (Ours)} \\
  \hspace{2mm}$k=2$
       & 4.081 & 9.28 & 19.56 & 2 & - \\
  \hspace{2mm}$k=4$
       & 3.297 & 9.38 & 13.06 & 4 & - \\
  \hspace{2mm}$k=6$
       & \textbf{3.192} & 9.43 & 12.58 & 6 & - \\
  \hspace{2mm}$k=8$
       & 3.199 & 9.44 & 12.68 & 8 & - \\
  \hspace{2mm}$k=12$
       & 3.207 & \textbf{9.52} & \textbf{12.17} & 12 & - \\
  \bottomrule
  \toprule
  \multicolumn{6}{l}{\textbf{Diffusion-based methods}}\\
  \midrule
DDIM & 4.67 & - & - & 50 & Ancestral \\
PNDM       & 4.61 & - &- & 20 & PNDM \\
DPM-Solver & 5.28 & - & - & 12 & RK45 \\
  \bottomrule
  \end{tabular}
\label{tab:c10_main}
\end{table}


%% file: tables/celeba_main.tex
\begin{table}[t]
  \small
  \centering
  \caption{
  The image generation result for CelebA-$64\times 64$ and LSUN-Church-$256\times 256$ datasets. 
  The KID measure is divided by $10^4$ and $10^2$ for CelebA and LSUN-Church datasets, respectively.
  E-M stands for Euler-Maruyama solver.}
  \begin{tabular}{lrrrl}
  \toprule
  \multicolumn{5}{l}{\textbf{CelebA $64\times 64$}} \\
  \midrule
  \multicolumn{5}{l}{\textbf{$k$-SeqRF} (Ours)} \\
  \midrule
  Method & FID & KID  & NFE  & Solver \\
  \midrule
  1-RF              & 2.327 & 12.95 & 113 & RK45-$10^{-4}$\\
  2-RF              & 5.841 & 38.01 & 50 & RK45-$10^{-4}$ \\
                    & 8.457 & 55.84 & 2 & Euler \\
  \midrule
  \multicolumn{5}{l}{\textbf{$k$-SeqRF} (Ours)} \\
  \hspace{2mm}$k=2$ & 4.672 & 29.28 & 2  & Euler \\
  \hspace{2mm}$k=4$ & 5.322 & 38.02 & 4  & Euler \\
  \hspace{2mm}$k=2$ & 3.878 & 25.56 & 52 & RK45-$10^{-3}$\\
  \hspace{2mm}$k=4$ & 5.268 & 36.98 & 68 & RK45-$10^{-3}$ \\
  \bottomrule
  \toprule
  \multicolumn{5}{l}{\textbf{Diffusion-based Methods}} \\
  \citep{song2021ddim} & 6.77 & - & 50 & Ancestral \\
  \citep{ho2020ddpm} & 45.20 & - & 50 & E-M \\
  DiffuseVAE & 5.43 & - & 50 & Euler \\
  \bottomrule
  \toprule
  \multicolumn{5}{l}{\textbf{LSUN-Church $256\times 256$}} \\
  \midrule
  Method & FID & KID  & NFE  & Solver \\
  \midrule
  1-RF (Pre-trained)&315.317 & 35.06 & 2 & Euler \\
                    &124.461 & 12.62 & 5 & \\
                    & 45.042 & 4.090 & 10 & \\
  2-RF              &104.417 & 11.99 & 2 & Euler\\
                    & 76.375 & 13.84 & 5 & \\
                    & 45.747 & 4.635 & 10 & \\
  \midrule
  \multicolumn{5}{l}{\textbf{$k$-SeqRF} (Ours)} \\
  \hspace{2mm}$k=2$ &104.417 & 11.99 & 2 & Euler\\
                    & 56.569 & 5.922 & 5 & \\
                    & 45.747 & 4.635 & 10 & \\
  \hspace{2mm}$k=4$ & 35.833 & 3.601 & 4 & Euler \\
                    & 30.883 & 2.959 & 5 & \\
                    & 28.066 & 2.632 & 10 & \\
  \bottomrule
  \end{tabular}
\label{tab:others_main}
\end{table}

%% file: main/06_conclusion.tex
\section{Conclusion and discussion}\label{sec:conclusion}

We introduce sequential reflow, a straightforward and effective technique for rectifying flow-based generative models.
This method initially subdivides the time domain into multiple segments and subsequently generates a joint distribution by traversing over partial time domains.
Through this, we successfully alleviate the global truncation error associated with the ODE solver.
Consequently, this process yields a straighter path, thereby enhancing the efficiency and speed of the sampling procedure.

Although we have focused on the generative modeling problem, our approach can be more broadened to other fields such as image-to-image translations, and latent generative models which can cover large-scale datasets.

\paragraph{Ethical aspects.}
Among the datasets we have used, CelebA datasets have biased attributes, such as containing more white people than black people, and more males than females.

\paragraph{Future societal consequences.}
Our paper proposes a new algorithm that contributes to the field of generative model research, especially flow-based generative models.
Flow-based models can be applied to the broader area of artificial intelligence, such as image-to-image translation that can be abused for Deepfake.

%% file: appendix/additional_experiments.tex
\appendix
\section{The overall algorithm of Sequential Reflow (SeqRF)}
\input{algorithms/srf}

\section{Experimental Details}\label{app:experiment}
\subsection{Dataset Description}
\paragraph{CIFAR-10}
The CIFAR-10 dataset is the image dataset that consists of 10 classes of typical real-world objects with $50,000$ training images and $10,000$ test images with $32\times32$ resolution.
In our experiments, we did not use the image labels and constructed unconditional generative models.
\paragraph{CelebA}
The CelebA dataset is the face dataset that consists of $202,599$ training images from $10,177$ celebrities with 40 binary attribute labels, with size $178\times218$.
In our experiment, we resized and cropped the image to $64\times64$ resolution to unify the input shape of the generative models.
Also, we did not use the attribute labels and constructed unconditional generative models.
\paragraph{LSUN-Church}
The Large-Scale Scene UNderstanding (LSUN)-Church(-Outdoor) dataset is the dataset that consists of the church images as well as the background which surrounds them.
This data consists of $126,227$ images with $256\times 256$ resolution.

\subsection{Details on Training Hyperparameters}
\input{tables/hyperparameter}
We report the hyperparameters that we used for training in~\cref{tab:hyper}.
We used a 32-bit precision floating point number for training all datasets.

\subsection{Performance Metrics}
We used the~\texttt {tfgan} package to measure the following metrics.
\paragraph{Frech\'et Inception Distance (FID)}
The Frech\'et Inception Distance (FID) uses the third pooling layer of the \texttt{Inceptionv3} network of the ground-truth image and generated images as the intermediate feature to interpret.
Let the mean and covariance of the ground-truth images be $(\mu_g,\Sigma_g)$ and those of the generated images as $(\mu_x,\Sigma_x)$.
Then the FID is defined by
\[
\mathrm{FID}(x)=\norm{\mu_x-\mu_g}_2^2 + \sqrt{\mathrm{Tr}\left( \Sigma_x+\Sigma_g-2(\Sigma_x \Sigma_g)^{\frac{1}{2}} \right)}.
\]
This represents the fidelity of the generated images, comparing the embedding distribution between the generated and ground-truth images.

\paragraph{Inception Score (IS)}
The inception score (IS) is defined by
\[
\exp
\left(
\bbE_x \bbE_{y|x}
\left[
\log\frac{p(y|x)}{p(y)}
\right]
\right),
\label{eq:is}
\]
where $x$ and $p(y|x)$ are the image data and the probability distribution obtained from $x$ using the~\texttt{InceptionV3}~\citep{szegedy2016inception} architecture.
The Inception Score represents the diversity of the images created by the model in terms of the entropy for ImageNet classes.

\paragraph{Kernel Inception Distance (KID)}
The kernel inception distance (KID) is the maximum mean discrepancy (MMD) metric on the intermediate feature space.
Let the distribution of the ground-truth and generated images be $p_g$ and $p_x$, respectively.
Then the KID is defined by
\[
\mathrm{KID}
&=
\bbE_{x_1, x_2\sim p_x}\left[ k(x_1, x_2) \right]
+
\bbE_{x_1,x_2\sim p_g}\left[ k(x_1, x_2) \right]
- 2
\bbE_{x_1\sim p_x, x_2\sim p_g}\left[  k(x_1, x_2) \right],
\]
where the similarity measure $k(x_1,x_2)$ is defined by
\[
k(x_1,x_2)
&=
(x_1\cdot x_2)^3.
\]
\section{Additional ablation studies}
\paragraph{EMA rate}
Since the flow matching and rectified flow is much harder to converge than the conventional diffusion model because of matching pairs of the independent distributions, we tuned the EMA rate in two ways:
\begin{itemize}
\item
We have increased the EMA rate to $0.999999$, higher than other models.
(Table) shows the ablation study on the EMA rate of the image synthesis quality from  CIFAR-10 datasets from baseline 1-rectified flow model for a 1000-step Euler solver.
\item
\textbf{Warm-up training policy.}\quad
Fixing the EMA rate high makes converging the model almost impossible because of the extremely high momentum.
So, we introduce the warmup phase with respect to the training step as
\[
\texttt{EMA\_rate=min((1 + step) / (10 + step), decay)}
\]
where \texttt{decay} and \texttt{step} are the given EMA rate and the training steps, respectively.
Introducing this warmup phase further improved the FID of our implementation in 1.3M steps and 128 batch size to \textbf{3.03} in CIFAR-10 dataset generation using a 1,000-step Euler solver.
\end{itemize}
\input{tables/c10_ema}

\subsection{Ablation on the number of reflow datasets}\label{app:sec:n_reflow}
\citet{liu2023rf} had reported the proper amount of number of reflow pairs for convergence should be at least 1M.
For completeness, we report how the image generation performance drops and overfits with a smaller number of reflow datasets.
With a small number of reflow pairs such as 10k, the performance is lower than the more-data cases and even becomes worse as the number of training steps elapses.

\begin{figure}[!ht]
\centering
\begin{subfigure}{0.49\textwidth}
\centering
\includegraphics[width=.99\linewidth]{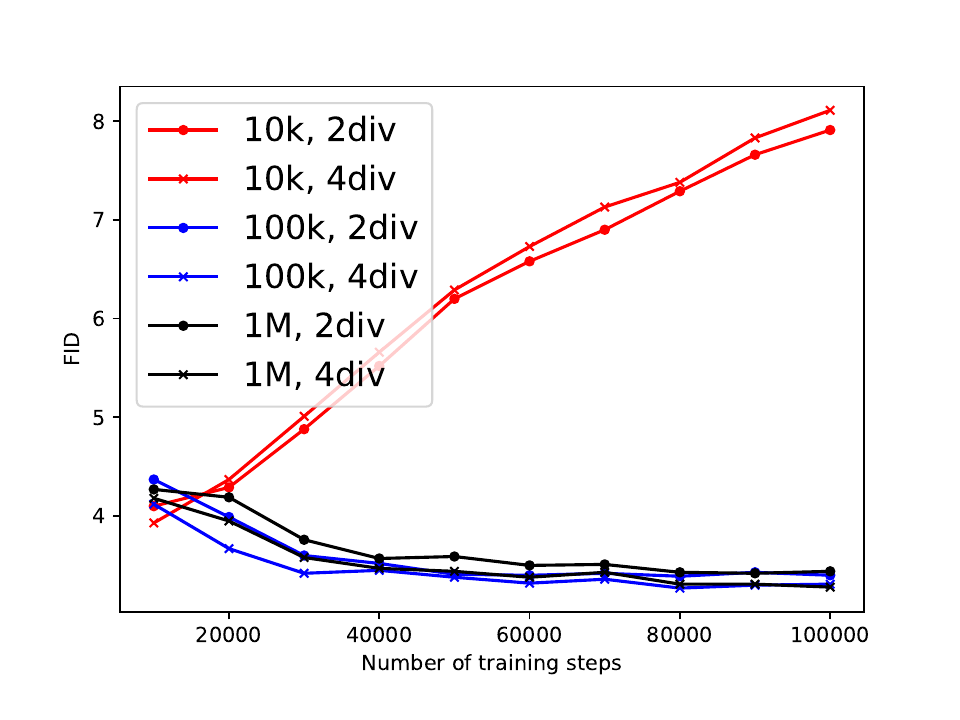}
\end{subfigure}
\begin{subfigure}{0.49\textwidth}
\centering
\includegraphics[width=.99\linewidth]{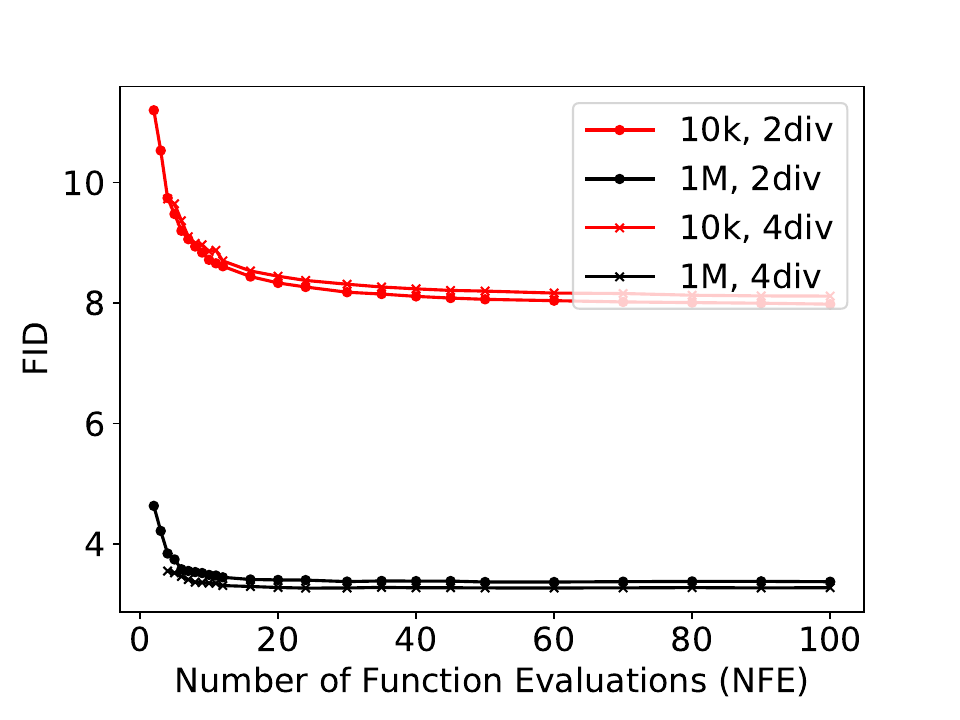}
\end{subfigure}
\caption{The generation performance of CIFAR-10 datasets with respect to the number of reflow datasets.
\textbf{Left:} Number of training steps vs. FID,
\textbf{Right:} Number of function evaluations vs. FID.}
\label{fig:test}
\end{figure}
\clearpage
\section{Example Images}
\subsection{Uncurated CIFAR-10 Images Generated by SeqRF-Distill Models}
\begin{figure*}[!ht]
\centering
\begin{center}
\begin{subfigure}[b]{0.8\textwidth}\label{fig:cifar}
\setlength{\lineskip}{0pt}
\centering
\includegraphics[width=\linewidth]{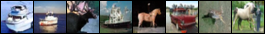}
\includegraphics[width=\linewidth]{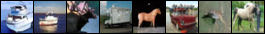}
\includegraphics[width=\linewidth]{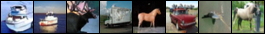}
\includegraphics[width=\linewidth]{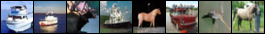}
\includegraphics[width=\linewidth]{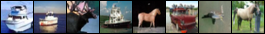}
\end{subfigure}
\end{center}
\caption{
Non-curated CIFAR-10 image synthesis result with $k$-SeqRF after distillation. From \textbf{up} to \textbf{down}: \{2, 4, 6, 8, 12\}-step Euler solver, each with an FID score of 
\{4.08, 3.30, 3.19, 3.20, 3.21\}, respectively.}
\label{fig:cifar_images}
\end{figure*}

\subsection{Uncurated SeqRF Datasets for Training CIFAR-10 Models.}
We provide examples of SeqRF (or the na\"ive reflow method) datasets used to train CIFAR-10 datasets in~\cref{fig:reflow}.
The images in the upper row represent the original images used to generate the source image in the middle row.
The images in the lower row represent the images generated by the \texttt{ODE-Solver}, initiated from the source image at the middle row.
\begin{figure}[!ht]
\centering
\begin{tabular}{cc}
\begin{subfigure}{0.45\textwidth}
\setlength{\lineskip}{-1pt}
\centering
\includegraphics[width=.99\linewidth]{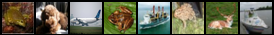}
\includegraphics[width=.99\linewidth]{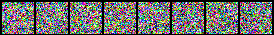}
\includegraphics[width=.99\linewidth]{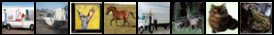}
\caption{No segment.}
\end{subfigure}
&
\begin{subfigure}{0.45\textwidth}
\setlength{\lineskip}{-1pt}
\centering
\includegraphics[width=.99\linewidth]{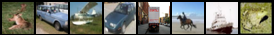}
\includegraphics[width=.99\linewidth]{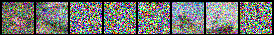}
\includegraphics[width=.99\linewidth]{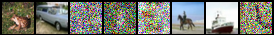}
\caption{2 segments.}
\end{subfigure}
\\
\begin{subfigure}{0.45\textwidth}
\setlength{\lineskip}{-1pt}
\centering
\includegraphics[width=.99\linewidth]{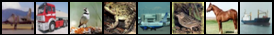}
\includegraphics[width=.99\linewidth]{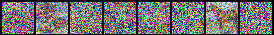}
\includegraphics[width=.99\linewidth]{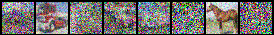}
\caption{4 segments.}
\end{subfigure}
&
\begin{subfigure}{0.45\textwidth}
\setlength{\lineskip}{-1pt}
\centering
\includegraphics[width=.99\linewidth]{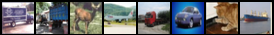}
\includegraphics[width=.99\linewidth]{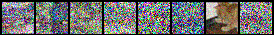}
\includegraphics[width=.99\linewidth]{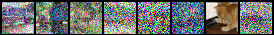}
\caption{6 segments.}
\end{subfigure}
\\
\begin{subfigure}{0.45\textwidth}
\setlength{\lineskip}{-1pt}
\centering
\includegraphics[width=.99\linewidth]{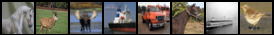}
\includegraphics[width=.99\linewidth]{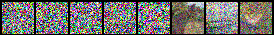}
\includegraphics[width=.99\linewidth]{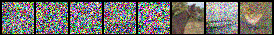}
\caption{8 segments.}
\end{subfigure}
&
\begin{subfigure}{0.45\textwidth}
\setlength{\lineskip}{-1pt}
\centering
\includegraphics[width=.99\linewidth]{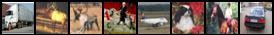}
\includegraphics[width=.99\linewidth]{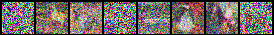}
\includegraphics[width=.99\linewidth]{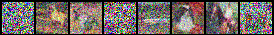}
\caption{12 segments.}
\end{subfigure}
\end{tabular}
\caption{
Sequential reflow dataset for CIFAR-10.
The \{upper, middle, lower\} rows represent the original data, source data (noisy original data), and destination data (ODE solver output from the noisy original data), respectively.
\cref{eq:joint_seqrf} describes the detailed procedure of constructing the (source, destination) pair.
}
\label{fig:reflow}
\end{figure}

\subsection{Uncurated LSUN-Church Images Generated by SeqRF Models}\label{app:sec:images}
\cref{fig:lsun_generated} demonstrates the images generated from rectified flow models that are pre-trained from~\citet{liu2023rf}, and from our proposed $k$-SeqRF models.
Except for $k=4$ case (4-SeqRF) which requires at least 4 steps to generate images, we presented images with NFE=(2, 5, 10), from top to bottom rows.
For reproducibility issues, all images are generated from the \textbf{same seed} of uniform Gaussian noises.
 
\begin{figure}[!ht]
\centering
\begin{tabular}{cc}
\begin{subfigure}{0.45\textwidth}
\setlength{\lineskip}{0pt}
\centering
\includegraphics[width=.99\linewidth]{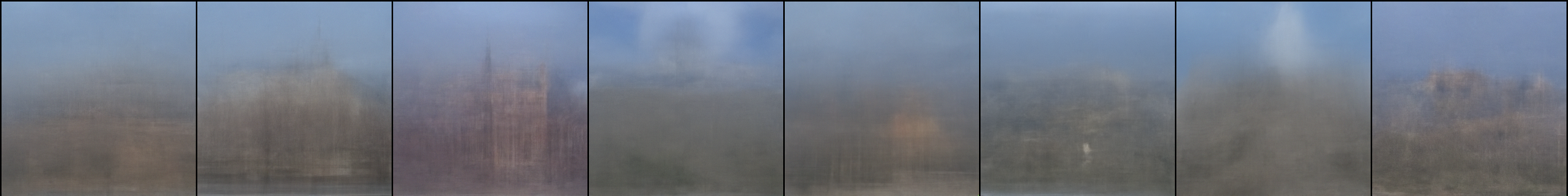}
\includegraphics[width=.99\linewidth]{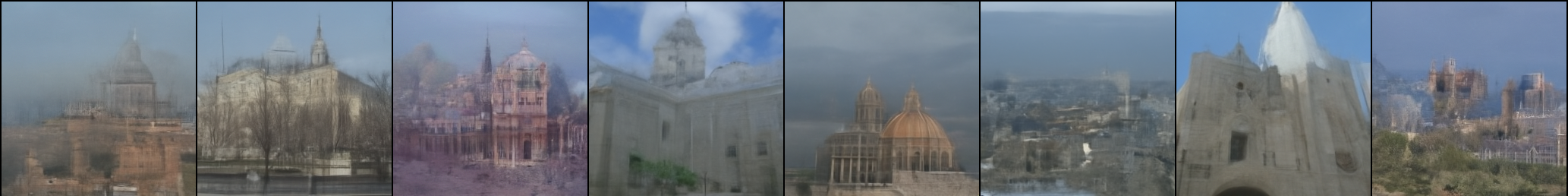}
\includegraphics[width=.99\linewidth]{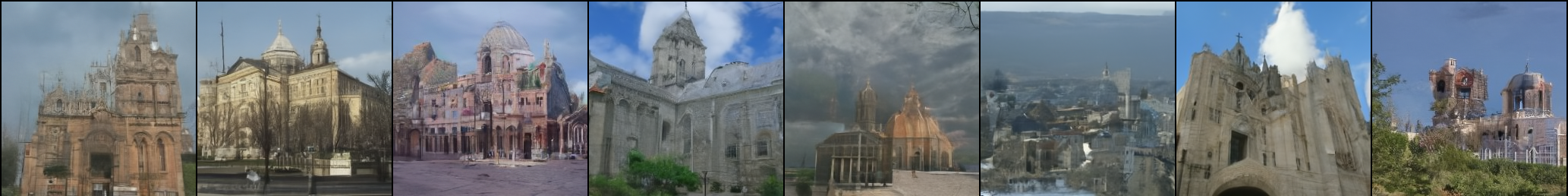}
\captionsetup{justification=centering}
\caption{1-RF~\citep{liu2023rf}.\\(NFE, FID)=\{(2, 315.3), (5, 124.5), (10, 45.0)\}.}
\end{subfigure}
&
\begin{subfigure}{0.45\textwidth}
\setlength{\lineskip}{0pt}
\centering
\includegraphics[width=.99\linewidth]{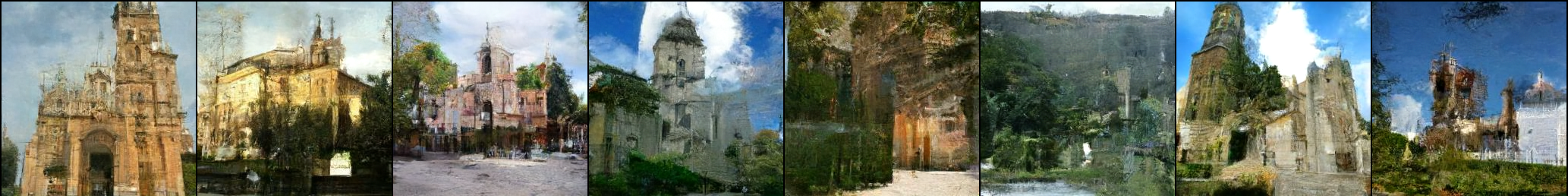}
\includegraphics[width=.99\linewidth]{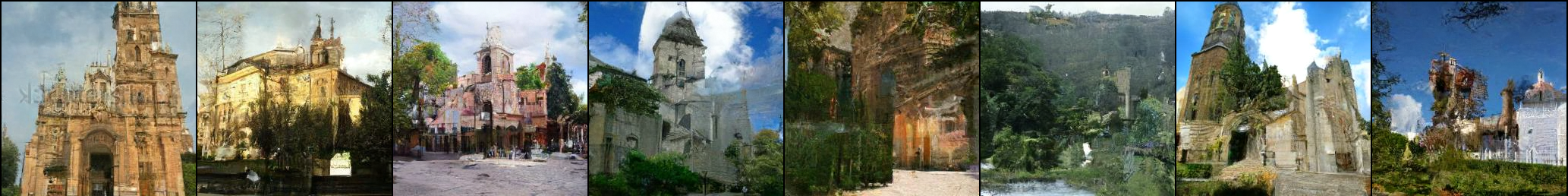}
\includegraphics[width=.99\linewidth]{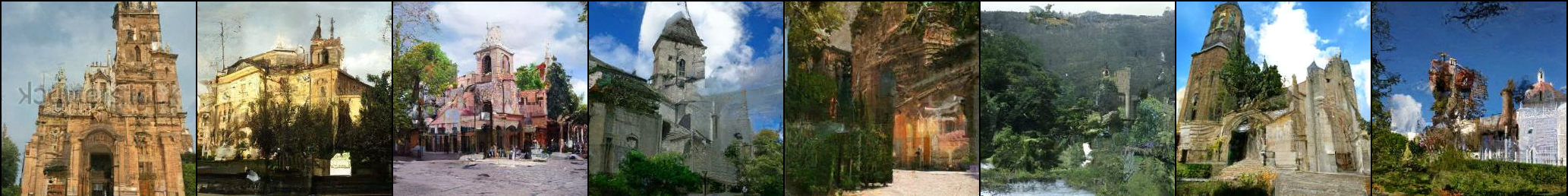}
\captionsetup{justification=centering}
\caption{2-RF (1-SeqRF).\\(NFE, FID)=\{(2, 123.8), (5, 76.4), (10, 54.6)\}.}
\end{subfigure}
\\
\begin{subfigure}{0.45\textwidth}
\setlength{\lineskip}{0pt}
\centering
\includegraphics[width=.99\linewidth]{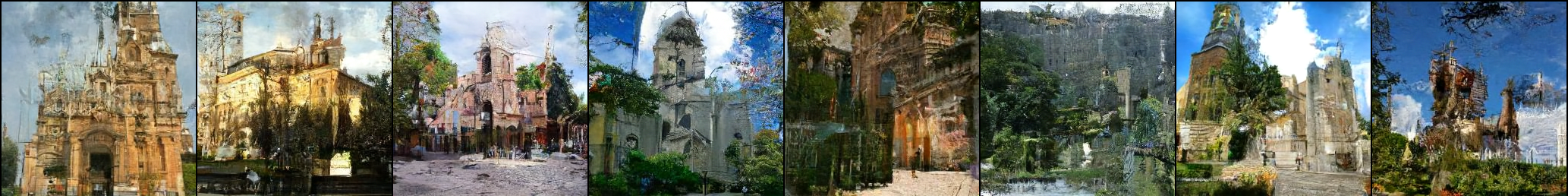}
\includegraphics[width=.99\linewidth]{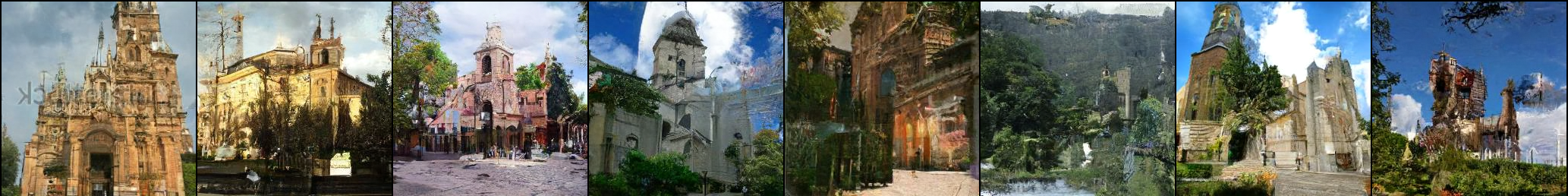}
\includegraphics[width=.99\linewidth]{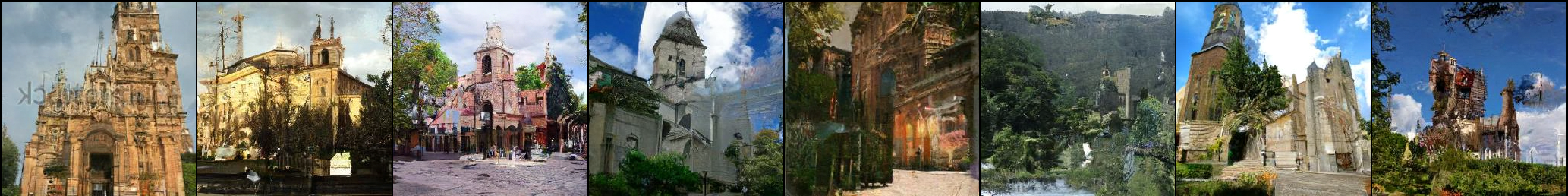}
\captionsetup{justification=centering}
\caption{\textbf{2-SeqRF} (Ours).\\(NFE, FID)=\{(2, 104.4), (5, 56.6), (10, 45.7)\}.}
\end{subfigure}
&
\begin{subfigure}{0.45\textwidth}
\setlength{\lineskip}{0pt}
\centering
\includegraphics[width=.99\linewidth]{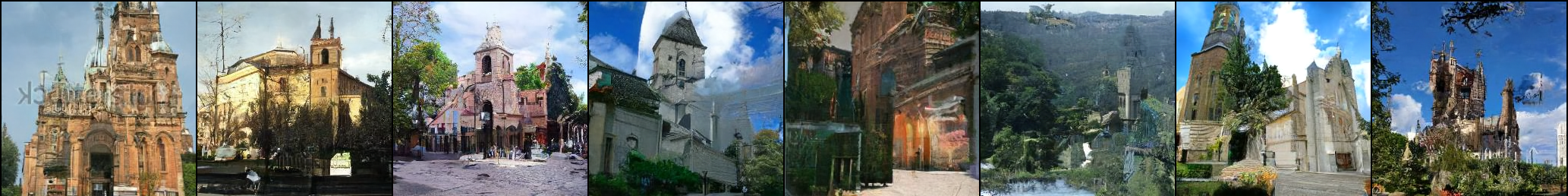}
\includegraphics[width=.99\linewidth]{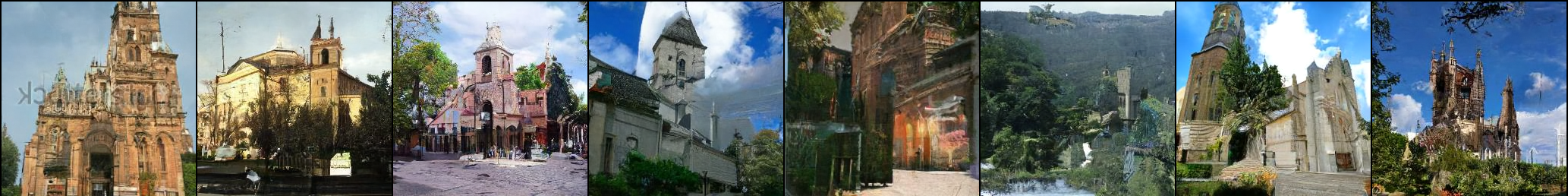}
\includegraphics[width=.99\linewidth]{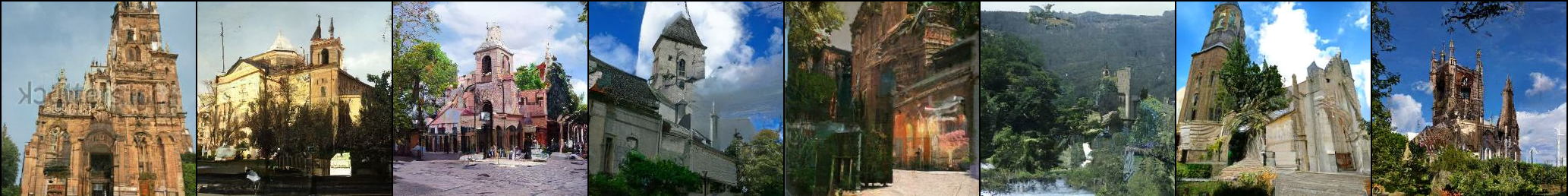}
\captionsetup{justification=centering}
\caption{\textbf{4-SeqRF} (Ours).\\(NFE, FID)=\{(4, 35.8), (5, 30.9), (10, 28.1)\}.}
\end{subfigure}
\end{tabular}
\caption{
Uncurated images generated by a few-step Euler solver from models learned by RF and SeqRF.
Except for the 4-SeqRF case (such that NFE=4 for the first row), the NFEs to generate images at three rows are 2, 5, and 10, respectively.
For reproducibility issues, all images are generated from the \textbf{same seed} of uniform Gaussian noises.
}
\label{fig:lsun_generated}
\end{figure}
\section{Equivalence of Flow Matching and Conditional Flow Matching}\label{app:fm_cfm}
In this section, we recap the equivalence of gradients of flow matching objective~\cref{eq:fm_objective} and conditional flow matching objective~\cref{eq:cfm_objective}.
For further information, refer to~\citet{lipman2023flow}, Theorem 2.

\begin{prop}[Equivalence of~\cref{eq:fm_objective} and~\cref{eq:cfm_objective}]
Let $p_t(x)>0,\forall x\in\bbR^D$ and $t\in[0,1]$, where $p_t$ is the marginal probability path over $x$.
Then, $\grad_\theta \calL_\text{FM}(\theta) = \grad_\theta \calL_\text{CFM}(\theta)$.
\end{prop}

\begin{proof}
First, we remind the two equations.
\[
\calL_\text{FM}(\theta)
&=
\bbE_{t,x_1}
\left[
\norm{u_t(x,\theta) - v_t(x)}_2^2
\right] \\
\calL_\text{CFM}(\theta)
&=
\bbE_{t,x_1,x_t|x_1}
\left[
\norm{u_t(x_t,\theta) - v_t(x_t|x_1)}_2^2
\right].
\]
From direct computation, we have
\[
\norm{u_t(x,\theta) - v_t(x)}_2^2
&=
\norm{u_t(x,\theta)}_2^2 - 2u_t(x)\cdot v_t(x) + \norm{v_t(x)}_2^2\\
\norm{u_t(x,\theta) - v_t(x)}_2^2
&=
\norm{u_t(x,\theta)}_2^2 - 2u_t(x)\cdot v_t(x|x_1) + \norm{v_t(x|x_1)}_2^2.
\]
Since $\bbE_{x}\left[ \norm{u_t(x,\theta)}_2^2 \right] = \bbE_{x_1, x|x_1}\left[\norm{v(t)}_2^2\right]$, the first term at the right-hand side is removed.
Since $v_t(x)$ and $v_t(x|x_1)$ is an analytically calculated form that is independent of $\theta$, the only remaining term is the dot products.
Then
\[
\bbE_{p_t(x)}\left[ u_t(x)\cdot v_t(x)\right]
&=
\int p_t(x) (u_t(x)\cdot v_t(x))\dee x \\
&=
\int p_t(x) \left(u_t(x)\cdot \int \frac{v_t(x|x_1)p_t(x|x_1)q(x_1)}{p_t(x)}\dee x_1\right)\dee x \\
&=
\int \int
u_t(x)q(x_1)\left( u_t(x) \cdot v_t(x|x_1)\right)
\dee x_1 \dee x \\
&=
\bbE_{x_1, x} \left[ u_t(x) v_t(x|x_1) \right]
\]
where the change of integration is possible since the flow is a diffeomorphism.
\end{proof}


\section{Sequential Straightness}\label{app:seq_straightness}
\subsection{Implementation details}
We measured the sequential straightness
\[
S_\text{seq}(\calZ)
&=
\sum_{i=1}^K \int_{t_{i-1}}^{t_i} \norm{\frac{Z_{t_i}^i - Z_{t_{i-1}}^i}{t_i - t_{i-1}} - \frac{\dee}{\dee t} Z_t^i}^2 \dee t
\]
by the following procedure.
\begin{enumerate}
\item[(1)]
Random draw the time interval $[t_{i-1},t_i]$ with $i\in[1:K]$, as in the SeqRF algorithm.
\item[(2)]
Sample the oracle input at time $t_i$ as $Z_{t_i}=(1-t_i) X_0 + t_i X_1$.
\item[(3)]
Then run the ODE solver to collect all sample within $[t_{i-1}, t_i]$.
\item[(4)]
$Z_{t_{i-1}} = \hat{X}_{t_{i-1}}$, that is, the terminal value of the reverse-time ODE solver at time $t_{i-1}$.
\item[(5)]
The vector field is calculated as $\frac{Z_{t_i} - Z_{t_{i-1}}}{t_i - t_{i-1}}$, and the derivative is defined as the difference between the values of two adjacent ODE solver timesteps, divided by the size of timesteps.
\end{enumerate}

\subsection{\mathwrap{Sequential straightness over $t$.}\label{app:sec:seq_str_plot_t}}
In addition to \cref{fig:seq_straightness}, we provide the sequential straightness over time for $k=\{2, 4, 8, 12\}$ in \cref{app:fig:seq_str_k}.
\begin{figure}[!ht]
\centering
\begin{tabular}{cc}
\centering
\includegraphics[width=0.3\linewidth]{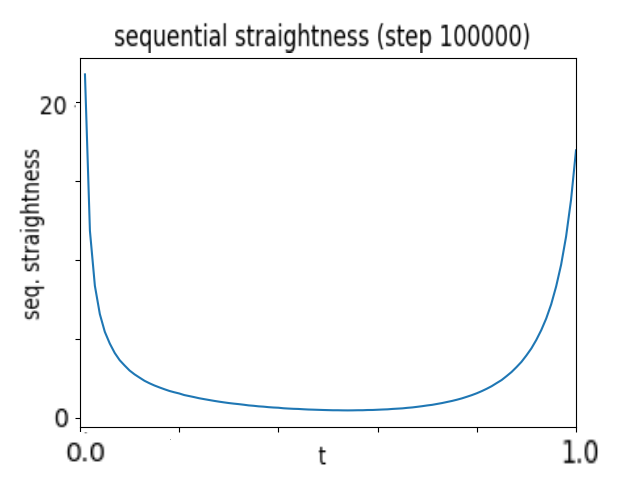}
&
\includegraphics[width=0.3\linewidth]{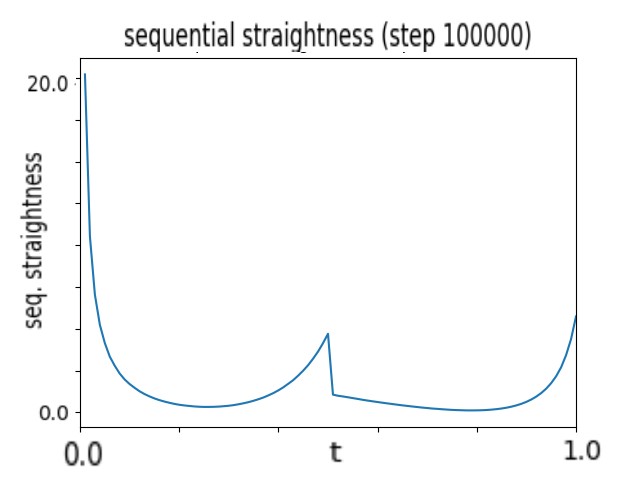}
\\
\includegraphics[width=0.3\linewidth]{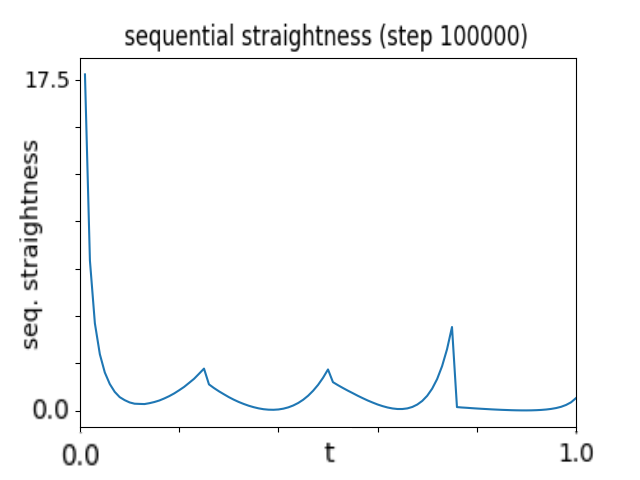}
&
\includegraphics[width=0.3\linewidth]{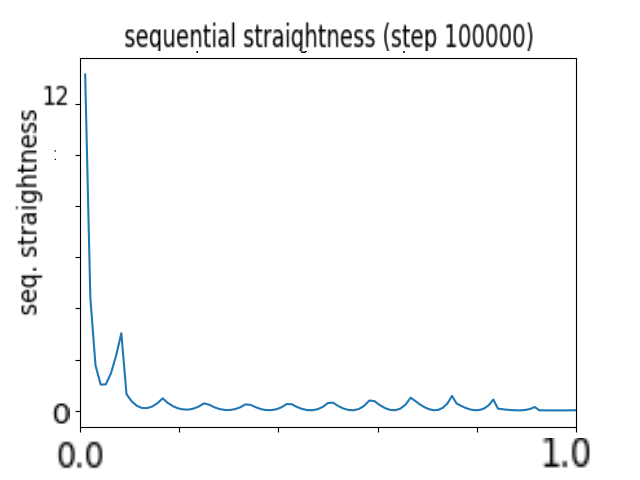}
\end{tabular}
\caption{
Sequential straightness results of $k$-SeqRF with $k=\{1, 2, 4, 12\}$, from upper left to lower right.
}
\label{app:fig:seq_str_k}
\end{figure}

\begin{figure}[!ht]
\centering
\begin{tabular}{ccc}
\centering
\includegraphics[width=0.3\linewidth]{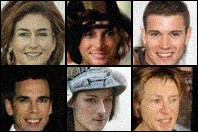}
&
\includegraphics[width=0.3\linewidth]{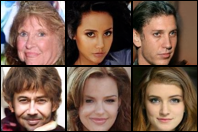}
&
\includegraphics[width=0.3\linewidth]{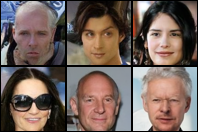}
\end{tabular}
\caption{
Uncurated random CelebA images chosen from random permutation sampled with $k$-SeqRF without distillation. From \textbf{up} to \textbf{down}: \{1, 2, 4\}-step Euler solver using $k$-SeqRF models.
}
\label{fig:celeb_images}
\end{figure}
\section{Proofs}

Before we take account of the error bound of the ODE with respect to the Lipschitzness of the ODE, we first define the local and global truncation error.
\begin{defn}[truncation error of the numerical method]
Let the ODE be defined as in~\eqref{eq:ode_ivp}.
Then the (local) truncation error at time $t$ with timestep $h$ is defined by
\[
\tau(t, h)
&=
\frac{x(t+h) - x(t)}{h} - f(x(t), t).
\label{eq:lte}
\]
The global truncation error at time $c\in[a,b]$ is defined by
\[
E(c) = x(c) - x_c.
\label{eq:gte}
\]
\end{defn}
\begin{defn}[Bound of local truncation error of the ODE solver]
Let the ODE be defined as in~\eqref{eq:ode_ivp}.
Then if the local truncation error is bounded by
\[
\tau(t,h)
\leq
\calO(h^{p})
\]
for some numerical solver \texttt{ODE-Solver$(x_{t_1}, t_1, t_2;v_\theta)$} for some $r$, i.e., 
\[
|\tau(t,h)\leq Kh^p\text{ for some }0<h\leq\left|h_0\right|
\]
for some finite $h_0$, then this solver is called to have \emph{order of accuracy} $p$.
\end{defn}
\subsection{\mathwrap{Proof of~\cref{thm:gte_bound}}\label{app:proof:gte_bound}}
Consider the general one-step method
\[
x_{t+h}=x_t + hv(x_t, t)
\]
where $v$ is a continuous function.
And Let $M_\textrm{sup}(t)$ be its supremum dilation, (e.g., maximum Lipschitz constant at time $t$), that is,
\[
\norm{v(x_t, t) - v(x_t ', t)}
\leq
M_\textrm{sup}(t)
\norm{x_t - x_t '}.
\]
for all $\{(t, x_t), (t', x_t')\}$ in a bounded rectangle
\[
D = \{(t, x): t \in [a,b], \norm{x - x_0}\leq C\}
\]
for some finite constant $C\geq 0$.

Then, the global truncation error $E$ is bounded by
\[
E(c)
&\leq
\int_{t=a}^c \tau(t)
\exp \left( \int_{t'=t}^c M_\textrm{sup}(t')\dee t' \right) \dee t.
\]
for all $c\in[a,b]$.

\begin{proof}
we first rewrite~\eqref{eq:lte} as
\[
x(t+h)
&=
x(t) + hf(x(t), t) + h\tau(t),
\]
then we get
\[
E(t+h)
&=
E(t)
+
h
\left[
f(t, x(t)) - f(t, x_t)
\right]
+ h\tau(t).
\]
Then by the Lipschitz condition, 
\[
E(t+h)
\leq
E(t) + hM(t)E(t) + h\tau(t),
\]
following that
\[
E(t+h)
\leq
(1 + hM(t))E(t) + h\tau(t).
\]
By induction through all timesteps $\{a, a+h, \cdots, a+Kh\}$,
\[
E(t+Kh)
&\leq
\sum_{i=0}^{K-1} h\tau(t + ih) \prod_{j=i}^{K-1} (1 + hM(t + jh))\\
&\leq
\sum_{i=0}^{K-1} h\tau(t + ih)  e^{\sum_{j=i}^{K-1} h M(t+jh)} . 
\]
When we take infinitesimal limit, i.e., $h\to 0$, with $K \to \frac{c-a}{h}$, then
\[
E(c)
&\leq
\int_{t=a}^c \tau(t)
\exp \int_{t'=t}^c M(t').
\]
As a special case, let $T_c=\max_{c\in[a,c]} \tau_c$ and $M(t)=M$ for all $t\in[a,c]$, then
\[
E(c)\leq \frac{T}{M}\left[\exp(M(c-a)) - 1\right].
\]
\end{proof}
\subsection{\mathwrap{Proof of~\cref{lem:dahlquist}}\label{app:dahlquist}}

\begin{proof}
Let the interval of a single step of the ODE solver be $h$.
Then the linear multistep method is defined as
\[
x_{t+h}
&=
\sum_{j=0}^p \alpha_j x_{t-jh} + h \sum_{j=-1}^p \beta_j v_t(x_{t-jh}, t-jh), \quad i\geq p
\]
where the interval of a single step is $h$, $t_{i}=t_0 + hi$ stands for $i$-th time step, and $v_t$ is the vector field, or the drift function.
Let the linear multistep method have the convergence order of $r$.
Then according to the \emph{Dahlquist equivalence theorem}, the global truncation error of the ODE solver has the order of $\calO(h^{r})$.
In case of the sequential reflow algorithm with the same~\gls{nfe}, the single interval of a single step is given as $\frac{h}{K}$ where $K$ is the number of equidistribted intervals that divide $[a,b]$ into time segments.
Then the global truncation error of the ODE solver of the sequential reflow algorithm is of $\calO\left(\left(\frac{h}{K}\right)^r\right)$.
As we have $K$ segments, the order of the global truncation error is $K\times\calO\left(\left(\frac{h}{K}\right)^r\right)=K^{1-r}\calO(h^r)$.
\end{proof}

\section{More related works}\label{app:sec:more_related_works}
\paragraph{Designing Probability Flow ODE}
As a problem of designing ODEs as probability path,~\citet{chen2018node} has opened up a new way by showing that an ODE can be learned with neural network objectives, and \citet{grathwohl2019ffjord} further applied this for the generative models.
Generating the invertible flow has been also achieved as architecture design~\citep{kingma2018glow} and from autoregressive model~\citep{papa2018maf,huang2018naf}.
Those early methods, however, are computationally inefficient, had memory issue, or not competitive in terms of synthesis quality in generative model literature.

\paragraph{Continuous-Time Generative Modeling}
As a family of continuous-time generative modeling by learning dynamics,~\citet{sohl2015diffusion} interpreted the generative modeling as the construction of the vector field of the stochastic processes.
Later on, ~\citet{ho2020ddpm,song2019ncsn,song2021scorebased} proposed efficient techniques for training this by interpreting the reverse stochastic process as the denoising model.
According to the diffusion models trained on SDEs, \citet{song2021ddim} proposed an efficient technique by sampling from a generalized non-Gaussian process, which results in a fast deterministic sampling speed.
Further, ~\citet{kingma2021vdm,huang2021variational} unified the framework of continuous-time models, including diffusion model and flow-based generative models.

%% file: algorithms/srf.tex
\begin{algorithm*}[!ht]
\caption{Training Sequential Rectified Flow for Generative Modeling}
\label{alg:train_srf}
\begin{algorithmic}
\REQUIRE{Data distribution $\pi_\text{data}$, Noise distribution $\pi_\text{noise}$, data dimension $D$, Number of divisions $K$, Flow network model $v_\theta:\bbR^D\to\bbR^D$, Time division $\{t_i\}_{i=0}^{K}$, $a=t_0<t_1<\cdots<t_{K-1}<t_K=b$.}
\vspace{0.5mm}
\STATE{\textbf{Stage 1. Pre-training rectified flow}}
\vspace{0.5mm}
\WHILE{Not converged}
  \STATE{Draw $(X_0,X_1)\sim\pi_\text{data}\times\pi_\text{noise}$, $t\in(0, 1)$.}
  \STATE{$X_t\gets (1-t) X_0 + t X_1.$}
  \STATE{Update $\theta$ to minimize $\bbE_{X_t,t}\left[\norm{v_\theta(X_t, t) - (X_1 - X_0)}_2^2\right]$}\COMMENT{Learning the flow network}
\ENDWHILE
\vspace{0.5mm}
\STATE{\textbf{Stage 2. Sequential reflow + Distillation with segmented time divisions}}
\vspace{0.5mm}
\WHILE{Not converged}
  \STATE{Sample $x_{t_i}=(1-t_k)x_a + t_k x_b$, $(x_a, x_b)\sim\pi_\mathrm{noise}\times\pi_\mathrm{data}$, $t_i, k\in[0:K-1]$}
  \STATE{Obtain $\hat{x}_{t-1}$ by running $\dee X_t = v_\theta(X_t, t) \dee t$ with an ODE solver from $t_i\to t_{i-1}$.}
  \IF{\texttt{Reflow}}
  \STATE{$x_s\gets (1-r) \hat{x}_{t_{i-1}} + r x_{t_i}$, $s=t' + (t_i-t_{i-1})r$, $r\sim\calU(0,1)$}
  \ELSIF{\texttt{Distill}}
  \STATE{$X_s\gets X_{t_i}$, $s=t$}
  \ENDIF
  \STATE{Update $\theta$ to minimize $\bbE_{x_s,r}\left[\norm{v_\theta(x_s, s) - \frac{x_s - x_{t_{k-1}}}{t_i - t_{i-1}}}_2^2\right]$}\COMMENT{Learning with sequential reflow}
\ENDWHILE
\end{algorithmic}
\end{algorithm*}

%% file: tables/hyperparameter.tex
\begin{table*}[!ht]
\centering
\caption{The hyperparameter used to train our model.}\vspace{-2mm}
\begin{tabular}{lccccc}
\\
\toprule
 & CIFAR-10 & CelebA & LSUN-Church\\
\midrule
Channels & 128 & 128 & 128 \\
Depth & 3 & 3 & 4 \\
Channel multiplier & (1, 2, 2, 2) & (1, 2, 2) & (1, 1, 2, 2, 2, 2, 2) \\
Heads & 4 & 4 & 4 \\
Attention resolution & 16 & 16 & 16 \\
Dropout & 0.15 & 0.1 & 0.1 \\
Batch size & 512 & 64 & 64 \\
Baseline steps & $800,000$ & $300,000$ & $600,000$ \\
Reflow steps & $100,000$ & $100,000$ & $200,000$ \\
Jitted steps & 5 & 5 & 5 \\
Reflow images & 1M & 200k & 1M \\
EMA rate & 0.9999 & 0.9999 & 0.999 \\
Learning rate (Adam) & $2\times 10^{-4}$ & $2\times 10^{-4}$ & $2\times 10^{-4}$\\
$(\beta_1, \beta_2)$ (Adam) & $(0.9, 0.999)$ & $(0.9, 0.999)$ & $(0.9, 0.999)$ \\
\bottomrule\label{tab:hyper}
\end{tabular}
\end{table*}

%% file: tables/c10_ema.tex
\begin{table*}[!ht]
\centering
\caption{The performance of the 1-rectified flow model with an Euler solver with 250 uniform timesteps, trained with batch size 128, 1.3M steps and different EMA rates.
The model does not converge with EMA rate 0.999999.}\vspace{-2mm}
\begin{tabular}{lccccc}
\\
\toprule
EMA rate & 0.999 & 0.9999 & 0.99999 & 0.999999 & Warm-up \\
\midrule
FID & 5.78 & 4.77 & 3.91 & - & \textbf{3.03} \\
\bottomrule
\end{tabular}
\end{table*}